\documentclass[11pt,a4paper]{article}
\usepackage[margin=0.75in]{geometry}
\usepackage[utf8]{inputenc}
\usepackage{cmap}
\usepackage[T1]{fontenc}
\usepackage{lmodern}
\usepackage{natbib}
\usepackage[colorlinks=true]{hyperref}
\usepackage{amssymb, amsthm, amsmath, amsfonts}
\usepackage{algorithm}
\usepackage{algpseudocode}
\usepackage{authblk, dsfont, mathrsfs, color, bm, enumitem, mathtools}
\usepackage{subfigure, graphicx, transparent}
\usepackage{dirtytalk}
\graphicspath{{figures/}}
\usepackage[dvipsnames]{xcolor}
\usepackage{comment}
\usepackage{tabularx}
\usepackage{booktabs}
\hypersetup{colorlinks, linkcolor={red!80!black}, citecolor={blue!80!black}, urlcolor={ForestGreen}}
\usepackage{multirow}
\usepackage[capitalize,noabbrev]{cleveref}
\usepackage{hyperref}

\theoremstyle{plain}
\newtheorem{theorem}{Theorem}[section]
\newtheorem{proposition}[theorem]{Proposition}

\theoremstyle{definition}

\newtheorem{assumption}[theorem]{Assumption}
\newtheorem{remark}[theorem]{Remark}

\usepackage{algorithm}
\usepackage{algpseudocode}
\usepackage{bbm}
\usepackage{amsthm}
\usepackage{booktabs} 
\usepackage{mathtools}
\usepackage{multirow}
\usepackage{amsmath}
\theoremstyle{plain} 
\usepackage{dsfont}
\usepackage[capitalize,noabbrev]{cleveref}

\DeclareMathOperator*{\argmin}{arg\,min}

\DeclareUnicodeCharacter{211D}{\ensuremath{\mathbb R}}

\newcommand{\rb}{\mathbb{R}}

\def \ds{\displaystyle}

\def \Vol{ {\operatorname{Vol} } }
\def \E{{\mathbb E}}
\def \P{{\mathbb P}}

\usepackage{pifont}

\usepackage[page]{appendix}

\makeatletter
\let\oldappendix\appendices

\renewcommand{\appendices}{%
  \clearpage
  \RestoreAddContentsLine 
  \renewcommand{\thesection}{\Roman{section}}
  \let\tf@toc\tf@app
  \addtocontents{app}{\protect\setcounter{tocdepth}{2}}
  \immediate\write\@auxout{%
    \string\let\string\tf@toc\string\tf@app^^J
  }
  \oldappendix
}%

\newcommand{\listofappendices}{%
  \begingroup
  \renewcommand{\contentsname}{\appendixtocname}
  \let\@oldstarttoc\@starttoc
  \def\@starttoc##1{\@oldstarttoc{app}}
  \tableofcontents
  \endgroup
}

\makeatother

\makeatletter
\newcommand{\RestoreAddContentsLine}{%
  \ifcsname hyper@anchor\endcsname
    \def\addcontentsline##1##2##3{%
      \addtocontents{##1}{%
        \protect\contentsline{##2}{##3}{\thepage}{\@currentHref}%
      }%
    }%
  \else
    \def\addcontentsline##1##2##3{%
      \addtocontents{##1}{%
        \protect\contentsline{##2}{##3}{\thepage}%
      }%
    }%
  \fi
}
\makeatother

\crefname{prop}{proposition}{propositions}
\Crefname{prop}{Proposition}{Propositions}
\crefname{ass}{assumption}{assumptions}
\Crefname{ass}{Assumption}{Assumptions}

\DeclarePairedDelimiterX{\inner}[2]{\langle}{\rangle}{#1, #2}

\title{Super-Level-Set Regression: Conditional Quantiles\\ via Volume Minimization}

\author[1]{Sacha Braun}
\author[$1, 2$]{Michael I. Jordan}
\author[1]{Francis Bach}
\affil[1]{Sierra team, Inria Paris, France \protect\\
\texttt{\{sacha.braun, francis.bach\}@inria.fr}}
\affil[2]{Departments of EECS and Statistics, UC Berkeley, USA \protect\\ \texttt{jordan@cs.berkeley.edu}}

\date{}

\begin{document}

\maketitle


\begin{abstract}

Constructing minimum-volume prediction regions that satisfy conditional coverage is a fundamental challenge in multivariate regression. Standard approaches rely on explicitly estimating the full conditional density and subsequently thresholding it. This two-step plug-in process is notoriously difficult, sensitive to estimation errors, and computationally expensive. One would like to instead optimize the region directly. Formulating a direct solution is challenging, however, because it requires minimizing a volume objective that is coupled with the conditional quantiles of the model's own estimation error. In this work, we  address this challenge. We introduce \emph{super-level-set regression} (SLS), a novel mathematical framework that successfully resolves this implicit coupling, allowing us to directly parameterize and optimize the geometric boundaries of the target conditional level sets. By bypassing full distribution estimation and leveraging flexible volume-preserving frontier functions, our approach natively captures complex, multimodal, and disjoint conditional structures end-to-end. Ultimately, SLS offers a new perspective on multivariate conditional quantile regression, replacing the restrictive assumptions of density-first methods with a direct geometric optimization strategy.

\end{abstract}


\section{Introduction}

In many statistical and machine learning applications, accurately characterizing the conditional distribution of a multivariate response variable, $Y \in \mathcal{Y} \subseteq \mathbb{R}^d$, given a feature vector $X \in \mathcal{X}$, is essential. However, in practice, learning the full, often highly complex, conditional density is both statistically challenging and computationally infeasible. For many decision-making tasks, we do not need the entire distribution; rather, it suffices to identify the highest density regions that capture a specific probability mass \citep{polonik1995measuring, izbicki2022cd}. These regions are formally known as conditional level sets. Traditionally, extracting these sets relies on a two-step plug-in approach: first estimating the global conditional density, and subsequently thresholding it \citep{hyndman1996computing}. This indirect route forces the model to expend representational capacity on areas of the space that are completely irrelevant to the target quantile. On the other hand, most existing direct estimation strategies are heavily constrained, requiring strong prior assumptions about the shape, location, or topology of the underlying data \citep{romano2019conformalized, feldman2023calibrated}.

In this work, we introduce \textit{super-level-set regression} (SLS), a novel framework designed to directly parameterize and optimize the boundaries of conditional level sets. Our approach represents a fundamental shift in perspective: rather than attempting to learn the full distribution, our method focuses exclusively and natively on the specific quantile of interest (see \Cref{fig:illustrative}). By optimizing the boundary directly, we provide an adaptive methodology that dynamically adjusts to the feature vector~$X$, allowing for highly flexible, covariate-dependent conditional level sets.

A key aspect of our method is the utilization of the region's volume as a direct optimization criterion. In our framework, the volume acts as a mechanism to naturally regularize the geometry of the prediction set, allowing us to bypass the rigid prior assumptions required by classical methods. By penalizing the volume of the set while enforcing a quantile constraint, these sets naturally align with the highest-density region of the data. Through the strategic design of the frontiers of these sets, our approach is capable of estimating level sets with highly intricate geometric structures. This includes handling strongly multimodal distributions, non-convex boundaries, and even complex shapes with internal voids or holes, while conditioning on $X$.

This approach of directly optimizing the region's volume poses an interesting challenge: it requires minimizing an objective that depends on a conditional quantile implicitly coupled with the model's own predictions. To optimize, one must differentiate through this instance-dependent quantile. While a variety of strategies exist for quantile-constrained optimization, such as differentiable sorting algorithms \citep{cuturi2019differentiable, blondel2020fast}, soft Lagrangian penalties \citep{cotter2019optimization}, or randomized smoothing \citep{berthet2020learning, grover2019stochastic}, these techniques are designed for marginal or batch-level quantiles. Adapting them to differentiate through \emph{conditional} quantiles that vary dynamically with the covariates $X$ is notoriously difficult, generally forcing a relaxation of the constraint into a weaker marginal guarantee \citep{braun2025minimum}. Another potential alternative is minimizing the conditional value-at-risk (CVaR) \citep{rockafellar2000optimization}; however, CVaR focuses solely on minimizing the expected loss of the tail. Consequently, it cannot be used to directly minimize arbitrary functions of the target quantile itself, which is exactly what our framework requires. 

As we will show, it is possible to circumvent this implicit coupling through a novel surrogate objective that averages the volume functional over a shrinking probability neighborhood, preserving the differentiability of the objective and capturing the conditional nature of the constraint, modulo finite-sample approximations.

\begin{figure}[t!]
    \center
\subfigure{\includegraphics[width=0.48\columnwidth]{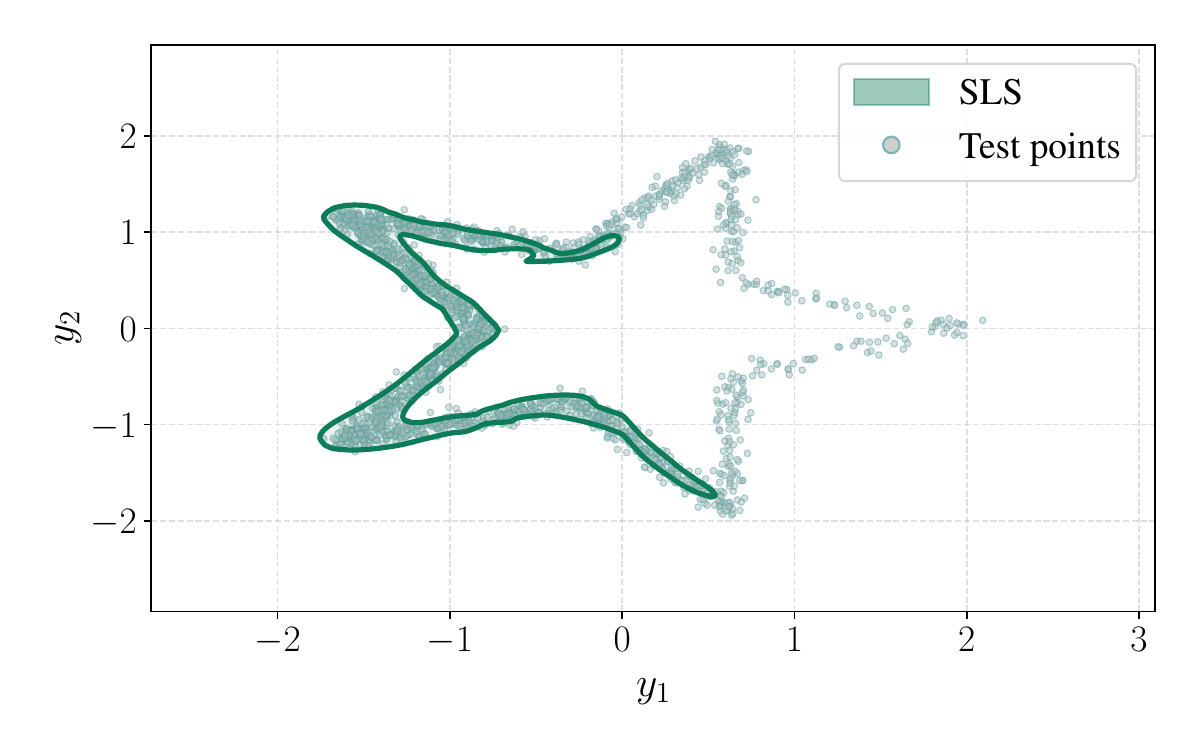}} ~
\subfigure{\includegraphics[width=0.48\columnwidth]{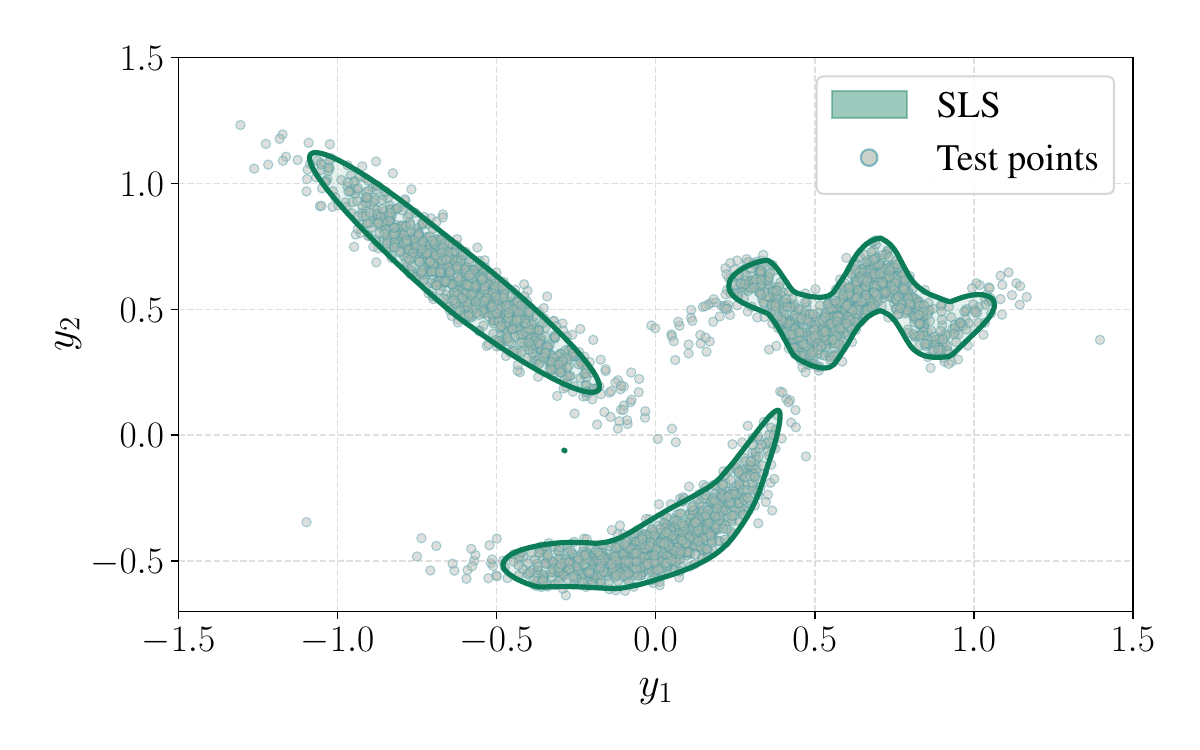}} 
    \vspace{-5mm}
    \caption{\textbf{SLS regression on synthetic 2D distributions}, sampling from $\P_{Y|X}$ for fixed $X$. \textbf{Left:} A single flow-based Mahalanobis frontier capturing an asymmetric, star-shaped density. Target and empirical coverage are both $70\%$. \textbf{Right:} A union of four flow components seamlessly adapting to a disjoint, three-mode distribution. The model successfully allocates mass despite the structural mismatch of components to modes, achieving an exact empirical coverage of $90\%$.}
    \label{fig:illustrative}
\end{figure}

\paragraph{Summary of contributions. }
Our main contributions are as follows:
\begin{itemize}[itemsep=0mm, topsep=0mm, leftmargin=5mm]
    \item We introduce a general optimization framework capable of minimizing objective functions that are implicitly coupled with the conditional quantiles of the model's own predictions (Proposition~\ref{prop:surrogate_convergence}).
    \item We propose \emph{SLS regression}, a novel approach that focuses exclusively on learning specific conditional level sets. This  bypasses the intermediate step of estimating the full conditional distribution and removes the need for restrictive prior assumptions about the data (\Cref{sec:level:sets:regression}).
    \item We provide an end-to-end framework where the geometric boundary of the level set is directly parameterized and optimized, offering a fundamentally new perspective on multivariate conditional quantile regression (\Cref{sec:frontier:examples}).
    \item We demonstrate that by carefully designing flexible parametric frontier functions, our method becomes highly adaptive to the feature vector $X$, capable of capturing complex conditional structures, including disjoint and multimodal shapes (\Cref{sec:experiments}).
\end{itemize}

\section{Related work}
\label{sec:related_work}

\paragraph{Highest density regions and predictive sets.}
The problem of constructing minimum-volume regions achieving a target coverage is solved in theory by the highest density region (HDR) \citep{hyndman1996computing, scott2005learning}. In the context of distribution-free regression, this is often formulated as finding highest predictive density (HPD) sets \citep{izbicki2022cd}. Classical methods for extracting HDRs or HPD sets typically follow a plug-in approach: they first train a model to approximate the full conditional density $p(y \mid X)$ and subsequently threshold this density to form regions \citep{dalmasso2020confidence, izbicki2019flexible, camehl2025superlevel, deliu2026alternative}. Normalizing flows \citep{dinh2014nice, papamakarios2021normalizing} are well suited for modeling complex distributions. In regression, conditional normalizing flows are predominantly trained via maximum likelihood to perform full conditional density estimation \citep{dinh2014nice, trippe2018conditional, english2025japan}, after which HDRs can be extracted post-hoc. While statistically sound, learning the entire conditional distribution is notoriously difficult to scale and calibrate \citep{scott2011multivariate}, often leading to computationally expensive and statistically inefficient prediction sets. Given a complex conditional density estimate, it is difficult to find the optimal conditional threshold yielding the HDR \citep{izbicki2022cd}. This often requires Monte Carlo sampling, and an increasing body of work has sought to find ways to compute these sets efficiently, or find distributions for which HDR can be computed in closed form \citep{wang2023probabilistic, plassier2024probabilistic, dheur2025unified, braun2025multivariate, dheur2025multivariate} or to rectify its levels post-hoc \citep{plassier2025rectifying}.

\paragraph{Multivariate quantiles and super-level set.}
Extending quantile regression to multivariate responses $Y \in \mathbb{R}^d$ is fundamentally challenging because there is no natural ordering in multiple dimensions. Optimal transport strategies have also emerged as go-to tools to rank multivariate responses \citep{thurin2025optimal, klein2025multivariate, ndiaye2025beyond}, but they fail to go beyond clustering for conditional coverage. Other existing approaches, such as directional quantiles or center-outward quantiles \citep{hallin2017multiple, del2024nonparametric}, provide geometrically appealing contours but often fail to guarantee that the enclosed regions correspond to the highest concentration of probability mass, especially for feature-dependent sets. 

\paragraph{Bypassing conditional density estimation.}
To circumvent the challenges of explicit conditional density estimation, a growing body of work attempts to learn prediction sets directly. This shift has been driven by the growing interest in conformal prediction \citep{papadopoulos2002inductive, vovk2005algorithmic, shafer2008tutorial}, which allows us to build confidence sets with marginal guarantees. For univariate responses, conformalized quantile regression (CQR) \citep{romano2019conformalized} and its variants directly estimate lower and upper bounds using the pinball loss. While highly efficient, CQR fundamentally enforces equal-tailed intervals. When the underlying distribution is skewed or multimodal, this restriction displaces the interval away from the true highest density regions, leading to unnecessarily wide and conservative sets. Recent advances, such as conformal thresholded intervals \citep{luo2025conformal}, attempt to threshold interquantile lengths to dynamically mimic local density. Quantile regression has also been adapted for multivariate settings, such as in the work of \citet{zhou2024conformalized} which produces hyper-rectangles, or that of \citet{feldman2023calibrated}, which requires learning multiple directional quantiles. Other methods frame region construction as a direct coverage allocation optimization problem \citep{sadinle2019least}. More recently, \citet{bach2025convex} derived a loss function aimed at minimizing set size while maintaining conditional coverage. However, its objective optimizes a global average across all possible coverage levels, meaning it cannot optimize the size for a specific quantile of interest. Closer to our minimum-volume approach, \citet{bars2025volume} focus on interval length minimization for univariate responses, and \citet{braun2025minimum} handle more complex forms induced by $p$-norms in the multivariate setting. However, neither of these works can handle multimodality and, more importantly, they relax the problem to one of marginal coverage. Our work addresses both of these limitations by explicitly targeting conditional coverage, while acknowledging that exact distribution-free guarantees for conditional coverage are unattainable and thus can only be approximated in practice.

\section{Super-level-set regression}
\label{sec:level:sets:regression}

Let $X \in \mathcal{X}$ denote a vector of covariates and let $Y \in \mathcal{Y} \subseteq \mathbb{R}^d$ be a continuous response variable. We assume that $(X,Y) \sim \P_{X,Y}$, where the joint distribution $\P_{X,Y}$ is unknown. Our objective is to construct conditional (super)-level sets for the density $\P_{Y|X}$, given $m$ independent and identically distributed (i.i.d.) samples, $(X_i, Y_i) \sim \P_{X,Y}$.

Assume that the conditional distribution $\mathbb{P}_{Y|X}$ admits a density $p(y \mid X)$ with respect to the Lebesgue measure. For a given probability level $\tau \in (0,1)$, the conditional super-level set is defined as
\begin{equation}
    \mathcal{A}(X) = \left\{ y \in \mathcal{Y} : p(y \mid X) \ge t_\tau(X) \right\},
\end{equation}
where $t_\tau(X)$ is the largest threshold satisfying
\begin{equation}
    \mathbb{P}\!\left(Y \in \mathcal{A}(X) \mid X\right) \ge \tau.
\end{equation}
Our goal is to learn the mapping $\mathcal{A}(\cdot)$. In practice, the true conditional density $p(y \mid X)$ is unknown and notoriously difficult to estimate with high fidelity \citep{scott2011multivariate}. Traditional strategies rely on negative log-likelihood minimization, but this approach generally requires strong prior assumptions and focuses on learning the entire distribution rather than a specific shape tailored for the target level~$\tau$.

The theoretical minimum-volume region achieving at least $\tau$ coverage, also known as the highest density region (HDR), can be shown by standard measure-theoretic arguments to be precisely the super-level set of the true conditional distribution, up to null sets~\citep{scott2005learning}. For every $X$, the region $\mathcal{A}(X)$ is the solution to the following constrained optimization problem, where $\Vol(\cdot)$ denotes Lebesgue measure:
\begin{equation}
\label{eq:min:vol:hdp}
    \inf_{\mathcal{B}(X)\subset\mathcal{Y}} \Vol(\mathcal{B}(X)) \quad \text{s.t.} \quad \P(Y\in\mathcal{B}(X)|X) \geq \tau\,.
\end{equation}
We propose bypassing explicit density estimation to instead search over a hypothesis space of real-valued functions $\mathcal{G} \subset \{G : \mathcal{X} \times \mathcal{Y} \to \mathbb{R}\}$, which we use to parameterize the level sets. We define a candidate confidence region as the sub-level set of one such function:
\begin{equation}
\label{eq:set:definition}
    \mathcal{C}_{G, q}(X) = \left\{ y \in \mathcal{Y} : G(X, y) \leq q(X) \right\}.
\end{equation}

In the following, we refer to a function $G \in \mathcal{G}$ as a \emph{frontier function}. Given a frontier function $G$, we denote by $\Vol_G(t, X)$ the volume of the set $\{y \in \mathcal{Y} : G(X,y) \leq t\}$ for a threshold $t > 0$.

To strictly satisfy the conditional validity constraint, the boundary threshold $q(X)$ must align with the conditional $\tau$-quantile of the frontier distribution, i.e., $q(X) \geq \text{Quantile}_\tau(\mathbb{P}_{G(X,Y) \mid X})$.

The problem of finding the optimal confidence region can thus be cast as minimizing the expected volume over the covariates:
\begin{equation} \label{eq:inf_problem}
    \inf_{G \in \mathcal{G}, q:\mathcal{X}\to\rb} \mathbb{E}_{X}\Big[ \Vol_G(q(X), X) \Big] \quad \text{s.t.} \quad q(X) \geq \text{Quantile}_\tau(\mathbb{P}_{G(X,Y) \mid X}) \quad \mathbb{P}_X\text{-a.s.}
\end{equation}
If the function class $\mathcal{G}$ is sufficiently expressive, solving \Cref{eq:inf_problem} exactly recovers the true super-level set $\mathcal{A}(X)$. Crucially, compared to conditional density estimation approaches that typically require restrictive structural assumptions about the data, the only implicit assumption we make is that the true super-level sets can be expressed in the form of \Cref{eq:set:definition} for some function $G \in \mathcal{G}$. For instance, with the simple choice $\mathcal{G} = \left\{ (x, y) \mapsto \|y - f(x)\|_2 \mid f: \mathcal{X} \to \rb^d \right\}$, we can only recover spherical level sets. More expressive frontier functions and a detailed analysis of this property are provided in \Cref{sec:frontier:examples}.

\subsection{Learning strategy}

Solving \Cref{eq:inf_problem} directly is intractable: the volume functional is generally non-differentiable with respect to $G$, and the quantile constraint $q(X) \geq \operatorname{Quantile}_\tau(\P_{G(X,Y) \mid X})$ implicitly couples $q$ and $G$. While the first challenge can be addressed by carefully designing the frontier set $\mathcal{G}$, the implicit link between $q$ and $G$ is harder to resolve. 

This problem is a bi-level optimization problem, where the learned parameters are $G(\cdot, \cdot)$ and $q(\cdot)$. The main issue is that for a fixed $q(\cdot)$, \Cref{eq:inf_problem} tends to reduce the volume uniformly across all samples. One reason for this behavior is that the samples $Y$ are only used to enforce the quantile constraint but do not regularize the volume minimization itself, which remains unconstrained. This leads to degenerate solutions, with the volume collapsing to zero.

To overcome the challenges of optimizing \Cref{eq:inf_problem} directly, we construct a surrogate objective that bypasses the strict quantile constraint by averaging the volume functional over a shrinking probability neighborhood around the $\tau$-quantile. The intuition is that when $G(X,Y) \approx q(X)$, we can replace $q(X)$ with $G(X,Y)$ in \Cref{eq:inf_problem}. 

To formalize the condition $G(X,Y) \approx q(X)$, we introduce a shrinking window around the target quantile. Let $\phi(n)$ and $\psi(n)$ be two positive sequences such that $\phi(n) \to 0$ and $\psi(n) \to 0$ as $n \to +\infty$, where $n$ denotes the number of training steps. We define the quantile neighborhood indicator $K_n(X, g)$ as $K_n(X, g) = \mathbbm{1}\left\{q_{\tau-\phi(n)}(X) \leq g \leq q_{\tau+\psi(n)}(X)\right\}$, where $q_\beta(X) = \operatorname{Quantile}_\beta(\P_{G(X,Y) \mid X})$ is the conditional $\beta$-quantile of the given frontier function $G$. By the law of total expectation and applying a change of variables to the probability level $\beta = F_{G|X}(g)$, we can rewrite the expectation of the unnormalized surrogate objective as:
\begin{equation}
\label{eq:quantile:interpretation}
    \mathbb{E}_{X,Y}\left[ K_n\big(X, G(X,Y)\big) \Vol_G\big(G(X, Y), X\big) \right] = \mathbb{E}_X\left[ \int_{\tau-\phi(n)}^{\tau+\psi(n)} \Vol_G\big(q_\beta(X), X\big) \, d\beta \right].
\end{equation}
This explicitly demonstrates that the surrogate objective performs a uniform volume minimization across all quantiles in the interval $[\tau-\phi(n), \tau+\psi(n)]$. Consequently, as the boundaries of this window shrink toward $\tau$, the surrogate increasingly minimizes the volume at the exact target quantile. The following proposition (proof in~\Cref{app:proofs}) makes this convergence formal.

\begin{proposition}[Convergence of the surrogate objective]
    \label{prop:surrogate_convergence}
    Consider the sequence of unconstrained optimization problems:
    \begin{equation}
        \label{eq:kn_reformulation}
        \inf_{G \in \mathcal{G}} J_n(G) \coloneqq \mathbb{E}_{X,Y}\left[ \frac{K_n\big(X, G(X,Y)\big)}{\psi(n)+\phi(n)} \Vol_G\big(G(X, Y), X\big) \right].
    \end{equation}
    Then, under some regularity assumptions (stated in Assumption~\ref{ass:regularity}), as $n \to +\infty$, the sequence of functionals $J_n(G)$ converges uniformly to the true objective $J(G) = \mathbb{E}_{X}\big[ \Vol_G(q_\tau(X), X) \big]$ over~$\mathcal{G}$. Furthermore, any limit point of a sequence of optimal solutions $G_n^* \in \argmin_{G \in \mathcal{G}} J_n(G)$ is an optimal solution to the constrained problem in \Cref{eq:inf_problem}.
\end{proposition}

This reformulation establishes that minimizing \Cref{eq:quantile:interpretation}, subject to a normalization constraint, is equivalent to the initial objective in \Cref{eq:inf_problem}. Crucially, this equivalence enables our learning procedure by naturally inducing an alternating optimization scheme. To learn the minimal-volume confidence set, we introduce two decay schedules, $\phi(\cdot)$ and $\psi(\cdot)$, and minimize the empirical analogue of \Cref{eq:quantile:interpretation}. We achieve this by alternately updating the frontier function $G$ while holding the quantile estimates $q$ fixed, and subsequently updating $q$ while holding $G$ fixed.

\begin{remark}
    \label{rem:generalization}
    Proposition~\ref{prop:surrogate_convergence} is not strictly tied to the volume functional $\Vol_G(t,X)$. The result naturally generalizes to any continuous objective function $h_G(t, X)$ that satisfies the regularity conditions outlined in Assumption~\ref{ass:regularity}. This point is further discussed in \Cref{app:subsec:function:of:quantile}. For instance, this framework can be used to learn a predictor $f:\mathcal{X} \to \mathbb{R}$ that minimizes the expected conditional median absolute deviation:
    \begin{equation*}
        \label{eq:median_example}
        \inf_{f:\mathcal{X} \to \mathbb{R}} \mathbb{E}_{X}\Big[ \operatorname{median}_{Y\sim \P_{Y|X}}\big(|Y-f(X)|\big) \Big]\,.
    \end{equation*}
\end{remark}

\paragraph{Quantile estimation via pinball loss.}
Since the true conditional quantile \mbox{$q_\beta(X) = \operatorname{Quantile}_\beta(\P_{G(X,Y) \mid X})$}  is not available in closed form, we learn it with a separate model $q_\beta : \mathcal{X} \to \mathbb{R}$ trained to minimize the pinball loss \citep{Steinwart_2011} while treating the frontier function $G$ as fixed, as in classical bi-level optimization \citep{bracken1973mathematical}. This update only affects the parameters for $q$, and treats $G$ as a non-differentiable object. This ensures that $q$ tracks the quantile of the current frontier distribution without interfering with the volume objective. Crucially, while this empirical approximation targets the conditional quantile, distribution-free finite-sample guarantees for exact conditional coverage are fundamentally impossible \citep{foygel2021limits}; consequently, our proposed method shares this theoretical limitation.

\paragraph{Initialization phase.} 
\label{para:warm:start}
To initialize the alternating optimization scheme and prevent early instability, we perform an initialization phase by modifying the surrogate objective in \Cref{eq:kn_reformulation}. Specifically, we replace the localized weighting term $K_n(X, G(X,Y))$ with $1$, yielding the simplified, unweighted objective $\mathbb{E}_{X,Y}[ \operatorname{Vol}_G(G(X, Y), X) ]$. The theoretical justification for this substitution is that computing the expectation over the full conditional distribution $Y \mid X$ is equivalent to integrating the $\alpha$-quantile volumes over all $\alpha \in [0,1]$. In other words, this unweighted expectation corresponds to $\int_0^1 \operatorname{Vol}_G(q_\alpha(X), X) d\alpha$, as in \citet{bach2025convex}. By removing the localized window $K_n$, we essentially optimize $G$ to minimize the average volume across all possible coverage levels simultaneously, rather than strictly focusing on the $\tau$-quantile. This global minimization provides a robust, non-degenerate initialization for the frontier function $G$, establishing a strong functional representation before narrowing the focus to the specific $\tau$-quantile neighborhood via the shrinking window.

The bi-level optimization procedure yields our final approach, summarized in \Cref{alg:level:set:regression}, which we refer to as \emph{super-level-set regression}. We assume that $G$ and the quantile estimators $q$ are parameterized by $\theta$ and $\omega$, respectively.

\begin{algorithm}[h!]
\caption{SLS regression}
\label{alg:level:set:regression}
\begin{algorithmic}[1]
\Require Data $\mathcal{D} = \{(X_i, Y_i)\}_{i=1}^m$, parameterized frontier function $G_\theta$, parameterized quantile networks $q_\omega$, target level $\tau$, decay schedules $\phi(\cdot), \psi(\cdot)$, warm-up steps $n_0$, total iterations $T$, learning rates $\eta_\theta, \eta_\omega$, objective function $h_G(t, X)$ (Volume or log volume for SLS regression).
\vspace{4pt}
\For{$j = 1, \dots, T$}
    \State Sample a mini-batch $\mathcal{B} \subseteq \mathcal{D}$
    \If{$j \leq n_0$}
        \State $ \hat{\mathcal{L}}_{\theta} \gets \frac{1}{|\mathcal{B}|} \sum_{(X_i, Y_i) \in \mathcal{B}} h_{G_\theta}\big(G_\theta(X_i, Y_i), X_i\big)$
    \Else
        \State Calculate window boundaries: $\beta_{\text{low}} = \tau-\phi(j)$ and $\beta_{\text{high}} = \tau+\psi(j)$
        \State $ \hat{\mathcal{L}}_{\theta} \gets \frac{1}{|\mathcal{B}|} \sum_{(X_i, Y_i) \in \mathcal{B}} \mathbbm{1}_{\mathopen\{q_\omega^{(\beta_{\text{low}})}(X_i) \leq G_\theta(X_i,Y_i) \leq q_\omega^{(\beta_{\text{high}})}(X_i)\mathopen\}} h_{G_\theta}\big(G_\theta(X_i, Y_i), X_i\big)$
    \EndIf
    \State $\theta \gets \theta - \eta_\theta \nabla_\theta \hat{\mathcal{L}}_{\theta}$
    \vspace{4pt}
    \State $\hat{\mathcal{L}}_{\omega} \gets \frac{1}{|\mathcal{B}|} \sum_{(X_i, Y_i) \in \mathcal{B}} \sum_{\beta \in \{\tau-\phi(j), \tau, \tau+\psi(j)\}} \mathcal{L}_{\text{pinball}}^{(i)}(\omega; \beta)$ 
    \State $\omega \gets \omega - \eta_\omega \nabla_\omega \hat{\mathcal{L}}_{\omega}$
\EndFor
\end{algorithmic}
\end{algorithm}

\section{Examples of Frontier Functions}
\label{sec:frontier:examples}

We start by restricting our attention to frontier functions $G \in \mathcal{G}$ for which the volume of the sub-level set $\{y \in \mathcal{Y} : G(X,y) \leq t\}$ admits a closed-form expression as a function of the threshold $t$. We first consider the flow-based Mahalanobis frontier:
\begin{equation}
    \label{eq:flow_frontier}
    G_\theta(X, Y) = \big\| L_\theta(X) \big(T_\phi(Y; X) - \mu_\theta(X)\big)\big\|_2^2\, ,
\end{equation}
where $\theta$ and $\phi$ are learnable parameters. Here, $L_\theta(X)$ is a $d \times d$ lower triangular matrix with strictly positive diagonal entries and $\mu_\theta(X) \in \mathbb{R}^d$. The mapping $T_\phi(\cdot; X) : \mathcal{Y} \to \mathbb{R}^d$ is a conditional volume-preserving normalizing flow (see, e.g., \citealt{dinh2014nice}). By design, this flow is a diffeomorphism constructed to have a Jacobian determinant of exactly $1$. This critical property guarantees that the Lebesgue measure is strictly invariant under the transformation $T_\phi$, yielding the volume surrogate
\begin{equation*}
    \Vol(q(X), X) \propto \frac{q(X)^{d/2}}{\det(L_\theta(X))} \, .
\end{equation*}

Further discussion on losses induced by the frontier in \Cref{eq:flow_frontier} is provided in \Cref{app:subsec:single:flow}. If we were to set $T_\phi(Y; X) = Y$, the confidence regions induced by \Cref{eq:set:definition} would be simple rigid ellipsoids. The normalizing flow, however, allows us to apply highly nonlinear transformations to these ellipsoids while exactly preserving their volume in the latent space. This clarifies our previous theoretical assertion: if the function class $\mathcal{G}$ is sufficiently expressive, we can exactly recover the true optimal level sets by minimizing \Cref{eq:inf_problem}. By incorporating the transformation $T_\phi$, a natural question arises: what family of geometric sets can be exactly characterized by the sub-level sets of \Cref{eq:flow_frontier}? The following proposition formalizes the expressive power of this formulation.

\begin{proposition}
    \label{prop:moser_single}
    Let $\Omega \subset \mathbb{R}^d$ be a compact domain with a smooth boundary, and assume $\Omega$ is diffeomorphic to a closed ball. Then, there exists a volume-preserving diffeomorphism $T: \rb^d\to\rb^d$, a $d\times d$ lower triangular matrix $L$ with strictly positive diagonal entries, a vector $\mu\in\rb^d$ and a constant $q > 0$ such that the target set $\Omega$ is exactly recovered by:
    \[
    \Omega = \left\{ y \in \rb^d : \big\|L(T(y)-\mu)\big\|_2^2 \leq q \right\}.
    \]
\end{proposition}

\begin{remark}
\label{remark:ellipical:distribution:equivalence}
The initialization phase procedure (\Cref{para:warm:start}) yields an important theoretical consequence for the flow-based Mahalanobis frontier: the localized optimization phase using the shrinking window $K_n$ can sometimes be omitted entirely. If a volume-preserving flow $T_\phi$ maps $Y \mid X$ to an elliptically contoured distribution (e.g., a multivariate Gaussian) with center $\mu_\theta(X)$ and covariance matrix $(L_\theta(X)^\top L_\theta(X))^{-1}$, the sub-level sets of $G(X,Y)$ perfectly match the true minimum-volume sets for \emph{all} coverage levels. In this scenario, the unweighted initialization phase objective is perfectly specified, as it minimizes expected volume globally across all levels. Thus, the initialization phase, followed by an exact calibration of $q(X)$ to the target $\tau$-quantile, is mathematically sufficient to recover the optimal confidence region. An experiment illustrating the usefulness of the shrinking window is provided in~\Cref{app:sec:alternative_losses}.
\end{remark}

\subsection{Union of normalizing flows}
\label{subsec:union:flow}
While a single volume-preserving flow is highly expressive, it is fundamentally a diffeomorphism and thus preserves topological connectedness. Consequently, it cannot easily model disjoint, heavily multimodal regions without introducing extreme, physically or biologically implausible distortions in the mapping. A flexible alternative captures complex, disconnected spaces via the union of localized flow-based regions, conceptually mirroring the level sets of a transformed Gaussian mixture, but optimizing boundaries directly by taking the union of flow-based regions.

We define an ensemble of $K$ independent conditional normalizing flows $T_{\phi_k}(Y; X)$, along with their respective localized networks $\mu_{\theta_k}(X)$ and shape matrices $L_{\theta_k}(X)$. We define the individual frontier for the $k$-th component by using its local latent Mahalanobis distance:
\begin{equation}
    \label{eq:component_frontier}
    G_k(X, Y) =  \big\|L_{\theta_k}(X) \big(T_{\phi_k}(Y; X) - \mu_{\theta_k}(X)\big)\big\|_2^2.
\end{equation}

To ensure all flow components actively participate in covering the target distribution, and to avoid the ``dead centers'' pathology associated with a hard minimum, we aggregate the scaled base frontiers using a softmax-weighted average (which acts as a differentiable min operator):
\begin{equation}
\label{eq:softmin_frontier}
G_\beta(X, Y) = \sum_{k=1}^K w_{k} G_k(X, Y), \quad \text{where} \quad w_{k} = \frac{\exp\big(-\beta \cdot G_k(X, Y)\big)}{\sum_{j=1}^K \exp\big(-\beta \cdot G_j(X, Y)\big)},
\end{equation}
where $\beta > 0$ acts as an inverse temperature controlling the sharpness of the approximation. In practice, $\beta$ is increased via an annealing scheme so that the frontier function converges to the hard minimum $G_\infty(X, Y) = \min_k G_k(X, Y)$, which is then used to construct the final level set. 

A practical limitation of the smooth formulation in \Cref{eq:softmin_frontier} is that the exact volume $\Vol_{G_\beta}(t, X)$ is not generally available in closed form. Instead, we optimize a volume proxy $\Omega_{\Vol}$ defined as the sum of the individual component volumes. While this approximation is not exact when components overlap, in practice, the optimization dynamics naturally drive the components toward configurations with negligible overlap. Further discussion of the training procedure induced by the frontier \Cref{eq:softmin_frontier} is presented in \Cref{app:subsec:union:flow}.

The theoretical capacity of these minimum-based frontier to capture disjoint topologies is established in the following proposition, whose proof is a direct consequence of Proposition~\ref{prop:moser_single}.
\begin{proposition}
    \label{prop:moser_union}
    Let $\Omega = \bigcup_{k=1}^K \Omega_k \subset \mathbb{R}^d$ be a finite union of $K$ disjoint sets, where each $\Omega_k$ is a compact domain with a smooth boundary and diffeomorphic to a closed ball. Then, there exists $G_\infty\in \mathcal{G}$ from \Cref{eq:softmin_frontier} such that (with the $X$-dependency omitted for clarity):
    \[
    \Omega = \left\{ y \in \mathbb{R}^d : G_\infty(y) \leq q \right\}.
    \]
\end{proposition}

\section{Experiments}
\label{sec:experiments}

All of our experiments can be found and reproduced in the associated GitHub
repository\footnote{\url{https://github.com/ElSacho/super_level_sets_regression}}. Additional information, including the model architecture, learning procedure, and the choice of the shrinking window schedules, is provided in Appendices~\ref{app:architecture}~\&~\ref{app:sec:training:strategy}.

\subsection{Synthetic experiments}

We start by assessing the robustness of our approach across different synthetic scenarios. Illustrative examples without dependence on the covariates $X$ are provided in \Cref{fig:illustrative}, and additional figures can be found in \Cref{app:additional:experiments}. These experiments demonstrate that our framework allows us to directly recover complex geometric shapes.

\begin{figure}[t!]
\center
\subfigure{\includegraphics[width=0.48\columnwidth]{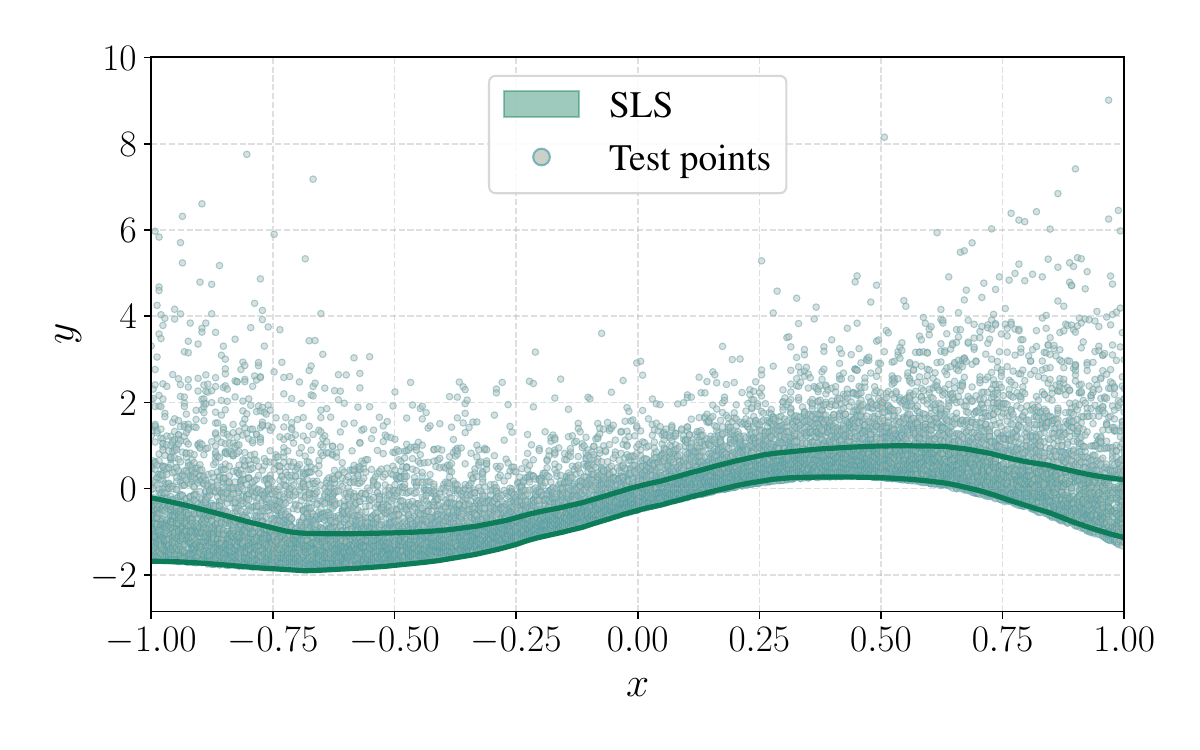}} ~
\subfigure{\includegraphics[width=0.48\columnwidth]{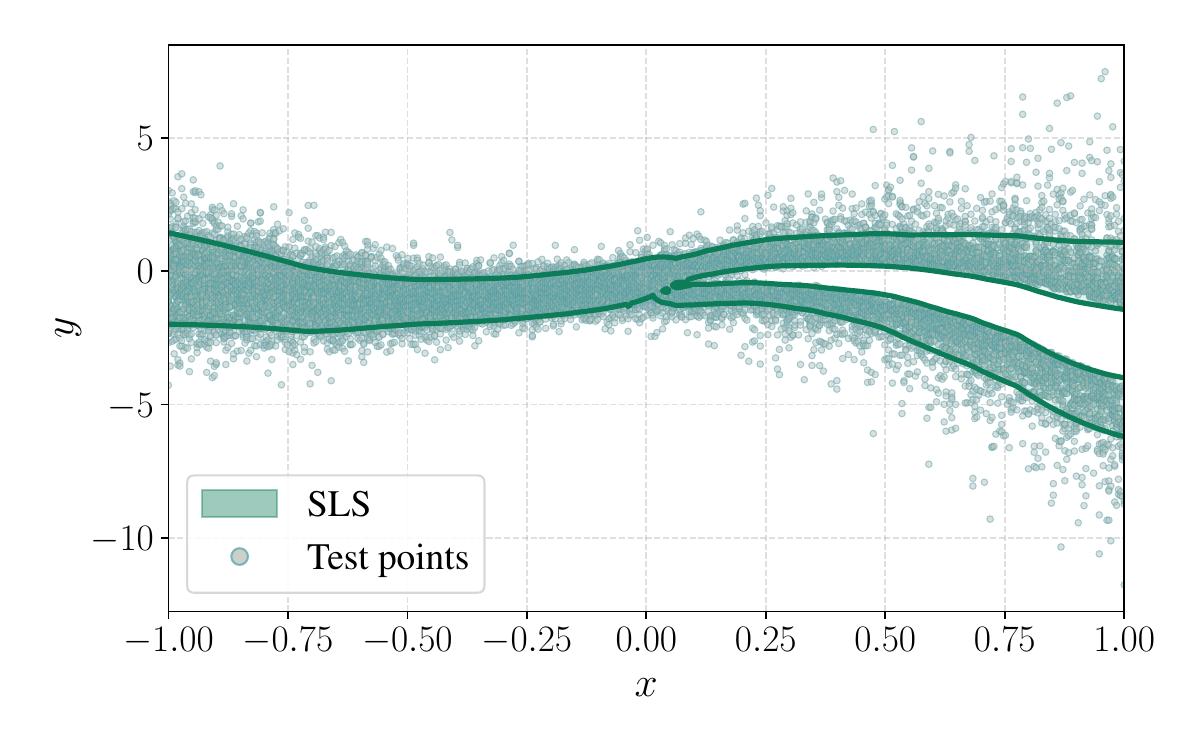}}
\vspace{-5mm}
\caption{\textbf{SLS regression on synthetic 1D conditional distributions.} \textbf{Left:} A single-component asymmetric exponential distribution evaluated at a target coverage of $30\%$. The shrinking window successfully corrects the sub-optimal centering caused by the initialization phase. \textbf{Right:} A mixture of exponentials transitioning from unimodal to bimodal, evaluated at a target coverage of $80\%$. The union of flows dynamically adapts to the changing topology.}
\label{fig:1d_synthetic}
\end{figure}

We now focus our analysis on scenarios featuring a strong dependence on the feature vector $X$. We begin with a univariate example where $X \sim \mathcal{U}([-1, 1])$, and the conditional distribution $Y|X$ follows a mixture of exponential distributions with either one or two components, as shown in \Cref{fig:1d_synthetic}. For the single-component case, we use the frontier $\mathcal{G} = \left\{ (x, y) \mapsto |y - f(x)| \mid f: \mathcal{X} \to \mathbb{R} \right\}$, which restricts the prediction sets to intervals. In this example, the unweighted initialization phase, which accounts for all possible volume sizes simultaneously, tends to push $f(x)$ towards the conditional median of $Y|X=x$. Because the exponential noise is highly asymmetric, the median is generally not the optimal center for the highest density region. This highlights the critical importance of the shrinking window, which enables the model to dynamically adjust and learn the precise level set for a specific target quantile. 

For the bimodal example, we employ the union of flows frontier (\Cref{eq:softmin_frontier}) with $K=3$ components. The model successfully identifies the unimodal structure for low values of $X$, and correctly switches to disjoint, multimodal sets as $X$ increases. Crucially, the flow components collaboratively align to form the correct union of intervals, effectively adapting even when the number of parameterized flows ($K=3$) exceeds the true number of modalities in the data.

We also evaluate our method in a multivariate setting featuring either heteroscedastic Gaussian noise contaminated with $10\%$ outliers, or heteroscedastic exponential noise. In both scenarios, we utilize the flow-based Mahalanobis frontier (\Cref{eq:flow_frontier}). For the Gaussian scenario, we restrict the conditional flow to the identity mapping, thereby limiting the model to recovering rigid ellipsoids. While standard maximum likelihood estimation (minimizing the Gaussian negative log-likelihood) would be heavily skewed by the presence of the outliers, our shrinking window approach learns to completely ignore data points that fall outside the target level set. Similarly, for the exponential noise scenario, our approach consistently and accurately captures the true level sets without being forced to model the entire asymmetric tail.

\begin{figure}[t!]
\center
\subfigure{\includegraphics[width=0.48\columnwidth]{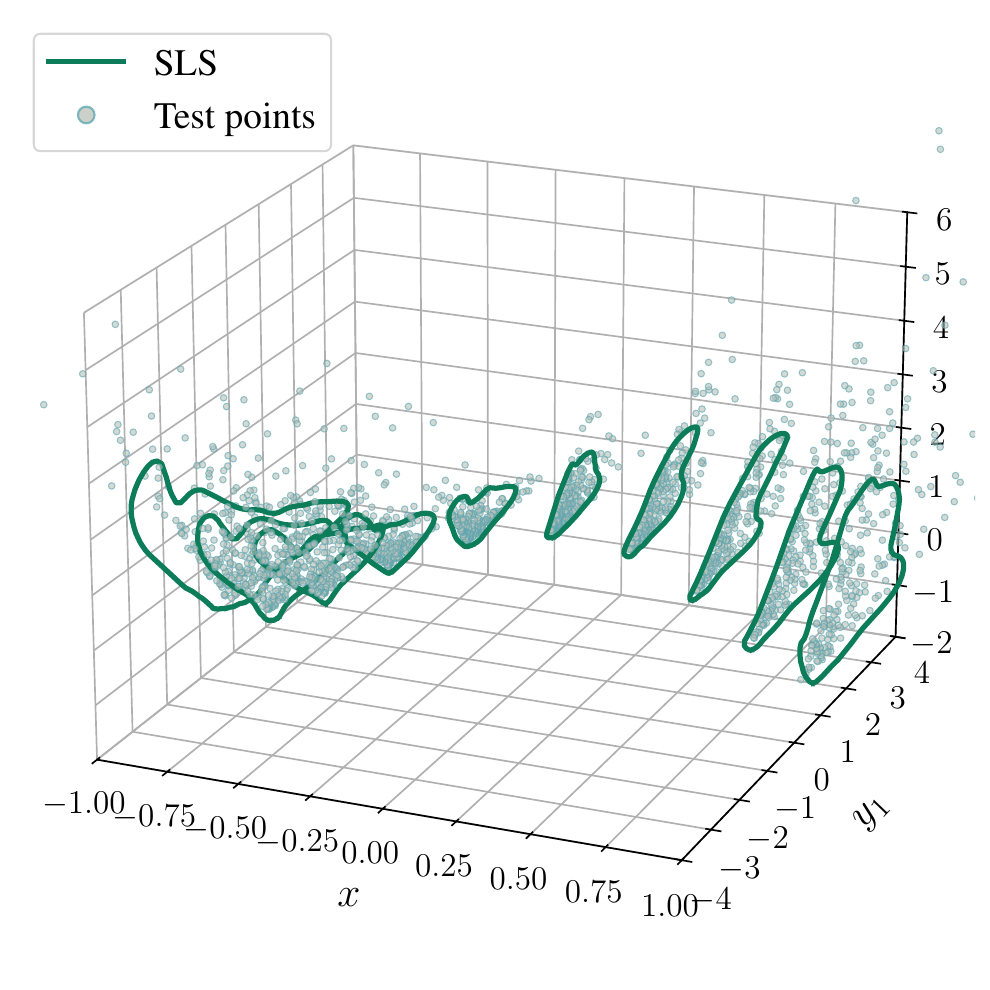}} ~
\subfigure{\includegraphics[width=0.48\columnwidth]{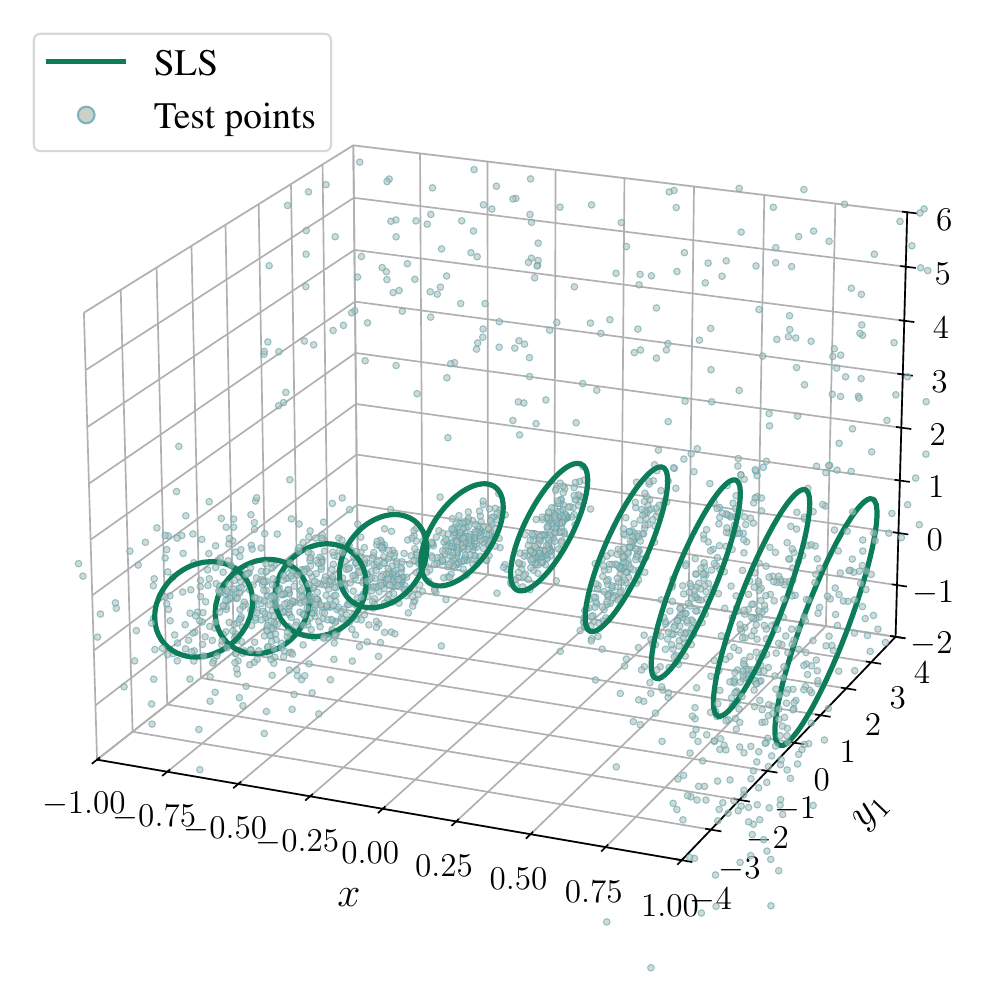}}
\vspace{-5mm}
\caption{\textbf{SLS on synthetic 2D conditional distributions} with a flow-based frontier. \textbf{Left:} Heteroscedastic exponential noise at a target coverage of $80\%$. \textbf{Right:} Heteroscedastic Gaussian noise contaminated with $10\%$ outliers, modeled with a rigid ellipsoid frontier (identity flow) at a target coverage of $60\%$. The model successfully ignores the outliers to capture the highest density region.}
\label{fig:2d_synthetic}
\end{figure}

\subsection{Real datasets}
For real-world experiments, we compare our approach against a conditional density estimation (CDE) baseline. The goal of this experiment is not to benchmark our strategy, but rather to establish empirically that the results are consistent with those of CDE. The CDE model minimizes the negative log-likelihood of a multivariate quantile function forecaster \citep{kan2022multivariate}, utilizing input convex neural networks \citep{brandon2017input} as implemented by \citet{dheur2025unified}. To guarantee marginal validity, we conformalize the CDE predictions with the C-HDR non-conformity scores \citep{izbicki2022cd}, and our approach with the score detailed in~\Cref{app:conformalizing}. We also compare our approach with the multivariate conformal prediction strategies targeting conditional coverage in \citet[Table 1]{dheur2025unified}. These are L-CP and C-PCP \citep{dheur2025unified} and CP2-PCP \citep{plassier2024probabilistic}. We also compare with PCP \citep{wang2023probabilistic} detailed in~\Cref{app:baselines}.

Evaluating these sets is challenging because the true data distribution is unknown. To measure performance, we plot the average volume (scaled by its $(1/d)$-th power, where $d$ is the output dimension) against the estimated conditional coverage deviation, $\mathbb{E}[|\mathbb{P}(Y \in C(X) \mid X) - \tau|]$. We approximate this deviation using the ERT metric \citep{braun2025conditional}.

The results, averaged across six datasets and ten random seeds, are presented in \Cref{fig:real:data}. Overall, SLS regression performs competitively with CDE, operating directly at the Pareto frontier of the inherent trade-off between region size and conditional coverage. When interpreting these results, it is crucial to analyze the trade-off simultaneously: while certain strategies may output smaller sets, they often do so at the cost of significantly degraded conditional coverage.

\begin{figure}[h]
    \center
\subfigure{\includegraphics[width=0.48\columnwidth]{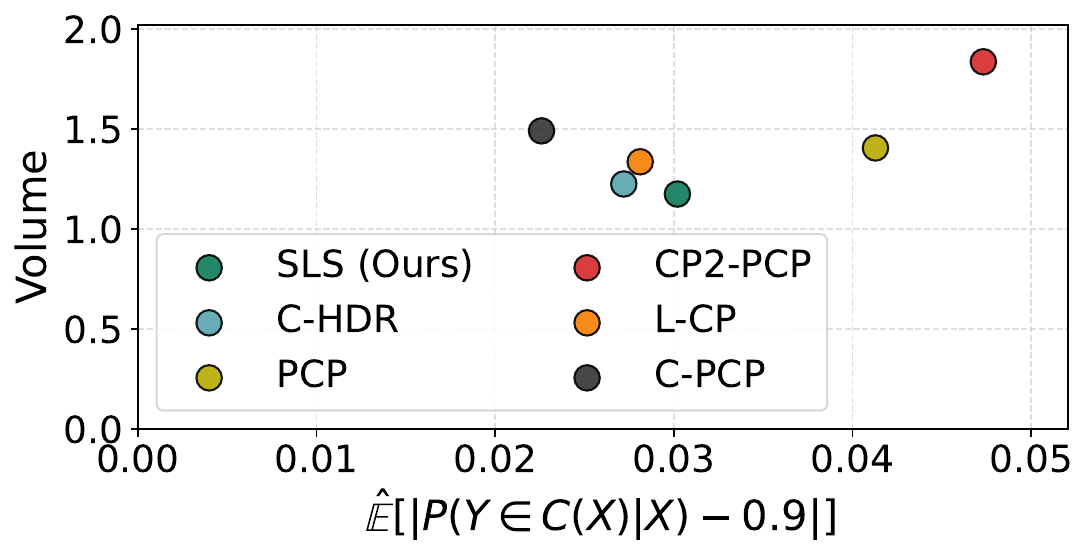}} ~
\subfigure{\includegraphics[width=0.48\columnwidth]{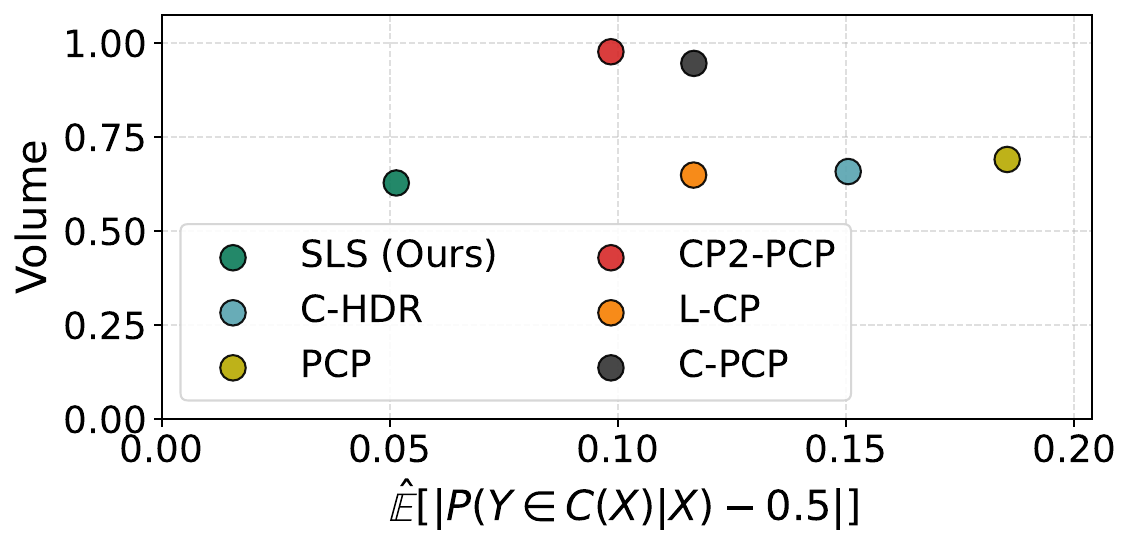}} 
    \caption{\textbf{Performance on real-world datasets.} Average scaled volume versus the estimated conditional coverage deviation. Closer to $(0,0)$ is better. \textbf{Left:} Target coverage 90\%. \textbf{Right:} Target coverage 50\%.}
    \label{fig:real:data}
\end{figure}

\section{Conclusion}
In this work, we introduced SLS regression, a novel framework that shifts the paradigm of highest density region estimation from an indirect density plug-in approach to direct geometric optimization. By developing a surrogate objective based on a shrinking probability window, we successfully resolved the notoriously difficult implicit coupling between volume minimization and conditional quantile constraints. A promising direction for future work is the design of other classes of frontier functions tailored to specific domains, such as incorporating physical priors, handling discrete or mixed-type responses, scaling the geometric optimization to highly structured, high-dimensional spaces.

\section*{Acknowledgements}

Authors acknowledge funding from the European Union (ERC-2022-SYG-OCEAN-101071601). Views and opinions expressed are however those of the author(s) only and do not necessarily reflect those of the European Union or the European Research Council Executive Agency. Neither the European Union nor the granting authority can be held responsible for them.

This publication is part of the Chair «Markets and Learning», supported by Air Liquide, BNP PARIBAS ASSET MANAGEMENT Europe, EDF, Orange and SNCF, sponsors of the Inria Foundation.

This work has also received support from the French government, managed by the National Research Agency, under the France 2030 program with the reference «PR[AI]RIE-PSAI» (ANR-23-IACL-0008).

Finally, the authors would like to thank Eugène Berta and David Holzmüller for their valuable feedback on the paper.


\bibliographystyle{Chicago}
\bibliography{bibliography}

\newpage
\onecolumn
\begin{appendices}
\listofappendices

\counterwithin{figure}{section}
\counterwithin{table}{section}

\crefalias{section}{appendix}
\crefalias{subsection}{appendix}

\newpage

\section{Proofs}
\label{app:proofs}

\begin{assumption}[Regularity of the hypothesis class]
    \label{ass:regularity}
    The function class $\mathcal{G}$ is a compact metric space. For $\mathbb{P}_X$-almost every $X \in \mathcal{X}$, the following hold uniformly for all $G \in \mathcal{G}$:
    \begin{enumerate}[itemsep=0mm, topsep=0mm, leftmargin=5mm]
        \item \textbf{Uniform density bound:} The conditional distribution of $G(X,Y)$ given $X$ admits a density $f_{G|X}(g)$, and there exists a strictly positive function $c(X) > 0$ such that $f_{G|X}(g) \geq c(X)$ in a neighborhood of the $\tau$-quantile $q_\tau(X) = \operatorname{Quantile}_\tau(\mathbb{P}_{G|X})$.
        \item \textbf{Uniform Lipschitz volume:} There exists a function $L(X)$ such that the mapping $g \mapsto \Vol_G(g, X)$ is $L(X)$-Lipschitz continuous in the aforementioned neighborhood of $q_\tau(X)$. Furthermore, the objectives $J_n(\cdot)$ and $J(\cdot)$ are continuous on the function class $\mathcal{G}$.
        \item \textbf{Integrability:} The true objective $J(G) = \mathbb{E}_{X}\big[ \Vol_G(q_\tau(X), X) \big]$ is finite, and $\mathbb{E}_X\left[ \frac{L(X)}{c(X)} \right] < \infty$.
    \end{enumerate}
\end{assumption}

\begin{proposition}[Convergence of the surrogate objective]
    \label{app:prop:surrogate_convergence}
    Suppose Assumption~\ref{ass:regularity} holds. Consider the sequence of unconstrained optimization problems:
    \begin{equation*}
        \inf_{G \in \mathcal{G}} J_n(G) \coloneqq \mathbb{E}_{X,Y}\left[ \frac{K_n\big(X, G(X,Y)\big)}{\psi(n)+\phi(n)} \Vol_G\big(G(X, Y), X\big) \right].
    \end{equation*}
    Then, as $n \to +\infty$, the sequence of functionals $J_n(G)$ converges uniformly to the true objective $J(G) = \mathbb{E}_{X}\big[ \Vol_G(q_\tau(X), X) \big]$ over $\mathcal{G}$. Furthermore, any limit point of a sequence of optimal solutions $G_n^* \in \argmin_{G \in \mathcal{G}} J_n(G)$ is an optimal solution to the constrained problem in \Cref{eq:inf_problem}.
\end{proposition}

\begin{proof}
The proof proceeds in two main steps. First, we establish the uniform convergence of the surrogate objective $J_n(G)$ to the true objective $J(G)$. Second, we rely on this uniform convergence and the compactness of the hypothesis class $\mathcal{G}$ to show the convergence of the minimizers.

For a fixed function $G \in \mathcal{G}$, we rewrite the surrogate objective $J_n(G)$ using the law of total expectation:
\[
J_n(G) = \mathbb{E}_X \left[ \mathbb{E}_{Y \mid X} \left[ \frac{K_n\big(X, G(X,Y)\big)}{\psi(n) + \phi(n)} \Vol_G\big(G(X,Y), X\big) \right] \right].
\]
Let $F_{G|X}$ denote the conditional cumulative distribution function of $g = G(X,Y)$ given $X$. By Assumption~\ref{ass:regularity}, $F_{G|X}$ is absolutely continuous in the neighborhood of $q_\tau(X)$. Expressing the inner conditional expectation as an integral with respect to $F_{G|X}$ yields:
\[
\frac{1}{\psi(n) + \phi(n)} \int_{q_{\tau-\phi(n)}(X)}^{q_{\tau+\psi(n)}(X)} \Vol_G(g, X) dF_{G|X}(g).
\]
Using the Lipschitz condition in Assumption~\ref{ass:regularity}, we get
\begin{align*}
    \sup_{G\in\mathcal{G}}|J_n(G) - J(G)| &\le \sup_{G\in\mathcal{G}} \mathbb{E}_X \left[ \frac{1}{\psi(n)+\phi(n)} \int_{q_{\tau-\phi(n)}}^{q_{\tau+\psi(n)}}\!\!\!\!\!\!\!\!\!\! |\text{Vol}_G(g,X) \!-\! \text{Vol}_G(q_\tau(X), X)| dF_{G|X}(g) \right] \\
    &\leq \sup_{G\in\mathcal{G}} \mathbb{E}_X \!\left[ \frac{1}{\psi(n)+\phi(n)} \int_{q_{\tau-\phi(n)}}^{q_{\tau+\psi(n)}} L(X)|g - q_\tau(X)| dF_{G|X}(g) \right] \\
    &\leq \sup_{G\in\mathcal{G}} \E_X \left[ L(X) \cdot\max \Big( \big| q_{\tau+\psi(n)}(X) \!-\! q_\tau(X) \big|, \big| q_{\tau-\phi(n)}(X) - q_\tau(X) \big| \Big) \right]\!,
\end{align*}
where we used $\ds\int_{q_{\tau-\phi(n)}}^{q_{\tau+\psi(n)}} dF_{G|X}(g)=\psi(n)+\phi(n)$.

Under the uniform density bound (Assumption~\ref{ass:regularity}.1), the inverse function theorem implies that the conditional quantile function is locally Lipschitz with respect to the probability level, with constant~$1/c(X)$. Consequently, for any sufficiently small perturbation $\epsilon$, we have \mbox{$|q_{\tau+\epsilon}(X) - q_\tau(X)| \leq \epsilon / c(X)$}. Applying this to the interval boundaries yields:
\[
\max \Big( \big| q_{\tau+\psi(n)}(X) - q_\tau(X) \big|, \big| q_{\tau-\phi(n)}(X) - q_\tau(X) \big| \Big) \leq \frac{\max(\phi(n), \psi(n))}{c(X)}.
\]
Substituting this bound into our supremum expectation, we obtain:
$$\sup_{G \in \mathcal{G}} |J_n(G) - J(G)| \leq \sup_{G \in \mathcal{G}} \mathbb{E}_X \left[ L(X) \frac{\max(\phi(n), \psi(n))}{c(X)} \right].$$
The term inside the expectation is now independent of $G$. Factoring out the deterministic sequences gives:
$$\sup_{G \in \mathcal{G}} |J_n(G) - J(G)| \leq \max(\phi(n), \psi(n)) \cdot \mathbb{E}_X \left[ \frac{L(X)}{c(X)} \right].$$
By Assumption~\ref{ass:regularity}.3, $\mathbb{E}_X [ L(X)/c(X) ]$ is finite. As $n \to +\infty$, we have $\phi(n) \to 0$ and $\psi(n) \to 0$, which forces the right-hand side to vanish. This establishes the uniform convergence:
$$\lim_{n \to +\infty} \sup_{G \in \mathcal{G}} |J_n(G) - J(G)| = 0.$$

For the second step, we start by showing the existence of $G_n^* \in \argmin_{G \in \mathcal{G}} J_n(G)$ for all $n$. 

Because $J_n$ is a continuous real-valued functional operating on the compact topological space $\mathcal{G}$, the extreme value theorem guarantees that it attains its global infimum. Consequently, the set $\argmin_{G \in \mathcal{G}} J_n(G)$ is non-empty, meaning there exists at least one optimal solution $G_n^* \in \mathcal{G}$ such that $J_n(G_n^*) = \inf_{G \in \mathcal{G}} J_n(G)$.

Let $(G_n^*)_{n \in \mathbb{N}}$ be a sequence of optimal solutions, such that $G_n^* \in \argmin_{G \in \mathcal{G}} J_n(G)$. Because $\mathcal{G}$ is a compact topological space, this sequence has at least one limit point. Let $G^*$ be such a limit point, meaning there exists a subsequence $(G_{n_k}^*)_{k \in \mathbb{N}}$ that converges to $G^*$. We aim to show that $G^* \in \argmin_{G \in \mathcal{G}} J(G)$.

By the definition of the minimizer, for any arbitrary function $G \in \mathcal{G}$, we have:
$$J_{n_k}(G_{n_k}^*) \leq J_{n_k}(G).$$
We take the limit of both sides as $k \to +\infty$. For the right-hand side, the uniform convergence of $J_n$ implies pointwise convergence, so $\lim_{k \to \infty} J_{n_k}(G) = J(G)$.

For the left-hand side, we apply the triangle inequality:
$$|J_{n_k}(G_{n_k}^*) - J(G^*)| \leq |J_{n_k}(G_{n_k}^*) - J(G_{n_k}^*)| + |J(G_{n_k}^*) - J(G^*)|.$$
The first term on the right vanishes as $k \to +\infty$ due to the uniform convergence of $J_n$ to $J$. The second term vanishes because $J$ is continuous and $G_{n_k}^* \to G^*$. 

Thus, $\lim_{k \to \infty} J_{n_k}(G_{n_k}^*) = J(G^*)$. Substituting these limits back into our inequality yields:
$$J(G^*) \leq J(G) \quad \forall \, G \in \mathcal{G}.$$
This concludes the proof that any limit point $G^*$ of the sequence of minimizers is an optimal solution to the true constrained problem.
\end{proof}

\begin{proposition}
    \label{app:moser:prop}
    Let $\Omega \subset \mathbb{R}^d$ be a compact domain with a smooth boundary, and assume $\Omega$ is diffeomorphic to a closed ball. Then, there exists a volume-preserving diffeomorphism $T: \mathbb{R}^d\to\mathbb{R}^d$, a $d\times d$ lower triangular matrix $L$ with strictly positive diagonal entries, a vector $\mu\in\mathbb{R}^d$ and a constant $q > 0$ such that the target set $\Omega$ is exactly recovered by:
    $$
    \Omega = \left\{ y \in \mathbb{R}^d : \big\|L(T(y)-\mu)\big\|_2^2 \leq q \right\}.
    $$
\end{proposition}

\begin{proof}
    Without loss of generality, we can fix $L=I_d$ and $\mu=\mathbf{0}_d$. Let $r > 0$ be the unique radius such that $\mathrm{vol}(\bar{B}(0, r)) = \mathrm{vol}(\Omega)$, i.e., $r = \bigl(|\Omega| / V_d\bigr)^{1/d}$, where $V_d$ denotes the volume of the unit $d$-ball, and set $q := r^2$.

    The existence of the mapping $T$ follows from the Dacorogna-Moser theorem \citep[Theorem 1]{dacorogna1990partial} extended to the whole space $\rb^d$ with control of support. Specifically, the proof is analogous to the construction in \cite[Theorem 1 \& Example 1]{teixeira2016dacorogna}, where the initial domain is $\Omega$ and the target domain is the standard Euclidean ball $\bar{B}(0, r)$ of equal volume. Because $\Omega$ is diffeomorphic to this closed ball and has the same volume, there exists a global volume-preserving diffeomorphism $T:\rb^d\to\rb^d$ such that $T(\Omega)=\bar{B}(0, r)$.
\end{proof}

\section{Alternative Volume Objectives and Theoretical Equivalences}
\label{app:sec:alternative_losses}

\subsection{Single-flow Mahalanobis frontier}
\label{app:subsec:single:flow}

\paragraph{Imposing a determinant constraint.}
An alternative approach to formulating the objective in \Cref{sec:frontier:examples} is to strictly constrain the determinant of the shape matrix. While not the exact method used across all of our experiments, enforcing $\det(L_\theta(X)) = 1$ almost surely for all $X$ acts as a scalar normalization that is seamlessly absorbed by the boundary threshold $q(X)$. Consequently, assuming the response space is $d$-dimensional, the volume of the candidate region becomes strictly proportional to $q(X)^{d/2}$. Since the mapping $z \mapsto z^{d/2}$ is strictly monotonically increasing for $z \ge 0$, and the optimization decomposes independently for every $X$, $\P_X$-almost surely; minimizing this expected volume is mathematically equivalent to minimizing the expected threshold $\mathbb{E}_X[q(X)]$ subject to the quantile constraint, a task natively handled by \Cref{alg:level:set:regression}.

\paragraph{Learning the $\log$-volume objective for stability.}
We empirically suggest minimizing the expected $\log$ volume rather than the true volume when the output dimension is too large. Due to the quantile constraint, those two objectives are equivalent, because the optimization decomposes independently for every $X$, $\P_X$-almost surely. However, the $\log$ volume is more stable with the dimension. 

\paragraph{Link with conditional density estimation.}
For elliptically contoured distributions, the highest density regions are natively ellipsoids. As noted in Remark~\ref{remark:ellipical:distribution:equivalence}, our unweighted initialization phase naturally pushes the mapped variable $T_\phi(Y; X)$ toward an elliptical distribution. Therefore, we can expect the loss function during this initialization phase to closely parallel maximum likelihood estimation. The following proposition formalizes this equivalence for a Gaussian distribution.

\begin{proposition}
Let $L(X) \in \mathbb{R}^{d \times d}$ be a parameterized lower triangular matrix, and let the conditional quantile of the scaled squared residuals be $q(X) = \operatorname{Quantile}_\tau(\|L(X)(T_\phi(Y; X)-f(X))\|_2^2)$. Minimizing the objective:
$$ \mathbb{E}_{X}\left[-\log\det(L(X)) + \frac{d}{2}\log q(X)\right] $$
is equivalent to minimizing the objective:
$$ \mathbb{E}_X\left[-\log\det(L(X)) + q(X)\right] $$
up to an additive constant independent of $L$.
\end{proposition}

\begin{proof}
Consider the integrand of the second objective, $J(L(X)) = -\log\det(L(X)) + q(X)$. Because a scalar multiple of a lower triangular matrix is still a lower triangular matrix, optimizing over $L(X)$ implicitly optimizes over any arbitrary scaling of the matrix. Let us scale $L(X)$ by a strictly positive scalar function $c(X) > 0$, defining $\tilde{L}(X) = c(X)L(X)$.

We evaluate the loss integrand for this scaled matrix. By the properties of the determinant for a $d \times d$ matrix, we have $\det(\tilde{L}(X)) = c(X)^d \det(L(X))$, which yields 
\begin{equation*}
    -\log\det(\tilde{L}(X)) = -d \log c(X) - \log\det(L(X))\, .
\end{equation*}
Furthermore, the squared Euclidean norm scales quadratically as $\|\tilde{L}(X)(T_\phi(Y; X)-f(X))\|_2^2 = c(X)^2 \|L(X)(T_\phi(Y; X)-f(X))\|_2^2$. Since $c(X)^2$ is a positive constant conditionally on $X$, it factors out of the quantile operator, giving the scaled quantile $\tilde{q}(X) = c(X)^2 q(X)$.

Substituting these into the loss integrand yields
\begin{equation*}
    J(\tilde{L}(X)) = -d \log c(X) - \log\det(L(X)) + c(X)^2 q(X)\, .
\end{equation*}
To find the optimal scaling for any fixed $L(X)$, we take the derivative of $J(\tilde{L}(X))$ with respect to $c(X)$ and set it to zero. This gives $-\frac{d}{c(X)} + 2c(X)q(X) = 0$. Solving for $c(X)$ subject to the constraint $c(X) > 0$, we obtain the optimal scaling factor $c^*(X) = \sqrt{\frac{d}{2q(X)}}$.

We now substitute this optimal scale back into the expression for $J(\tilde{L}(X))$:
\begin{equation*}
\min_{c(X) > 0} J(c(X)L(X)) = -d \log\left(\sqrt{\frac{d}{2q(X)}}\right) - \log\det(L(X)) + \left(\sqrt{\frac{d}{2q(X)}}\right)^2 q(X)\, .
\end{equation*}
Simplifying the first term gives $-\frac{d}{2} \log\left(\frac{d}{2q(X)}\right) = -\frac{d}{2} \log\left(\frac{d}{2}\right) + \frac{d}{2} \log q(X)$. The final term simplifies to exactly $\frac{d}{2}$. Combining these components, we obtain
$$ \min_{c(X) > 0} J(c(X)L(X)) = \left[ -\log\det(L(X)) + \frac{d}{2} \log q(X) \right] + \frac{d}{2}\left(1 - \log\left(\frac{d}{2}\right)\right). $$
The bracketed expression is precisely the integrand of the first objective. The remaining terms form a constant that depends only on the dimension $d$ and not on the data, the model $f(X)$, or the parameters of $L$. Taking the expectation over $X$, the two optimization problems differ only by this additive constant, completing the proof.
\end{proof}

To further contextualize this result, note that the first objective corresponds to minimizing the expectation of the log-volume of the estimated level sets. Because it is theoretically equivalent to the second objective, we can optimize it directly using the alternating scheme in \Cref{alg:level:set:regression}. During the initialization phase, our procedure minimizes the unweighted proxy:
$$ \mathbb{E}_{X,Y}\left[-\log\det(L(X)) + \|L(X)(Y-f(X))\|_2^2\right], $$
which is functionally equivalent to minimizing the negative log-likelihood of a multivariate Gaussian distribution with precision matrix $\Sigma(X)^{-1} = 2 L(X)^\top L(X)$.

Crucially, however, our overall framework fundamentally diverges from standard density estimation. Because we explicitly optimize for volume subject to a specific quantile constraint, our method isolates the target level set and ignores distant outliers that would otherwise skew standard maximum likelihood methods. This robustness is illustrated in \Cref{fig:comparison:gaussian}, where we fix the flow to the identity mapping to recover rigid ellipsoids and further expose this phenomenon. While minimizing the negative log-likelihood fails due to the outliers or a non-elliptical distribution, our shrinking window approach successfully ignores the irrelevant probability mass to find the optimal minimum-volume ellipsoid.

\begin{figure}[t!]
\center
\subfigure{\includegraphics[width=0.48\columnwidth]{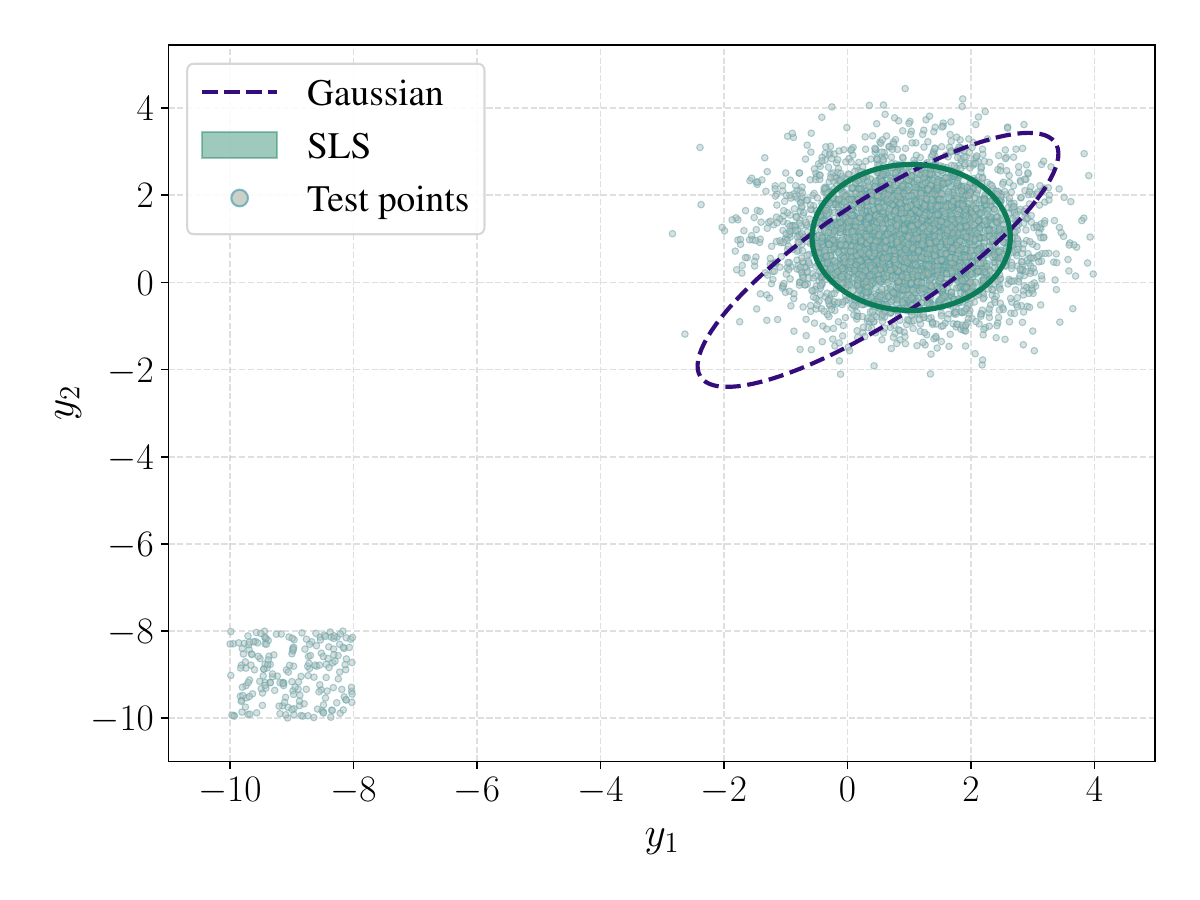}} ~
\subfigure{\includegraphics[width=0.48\columnwidth]{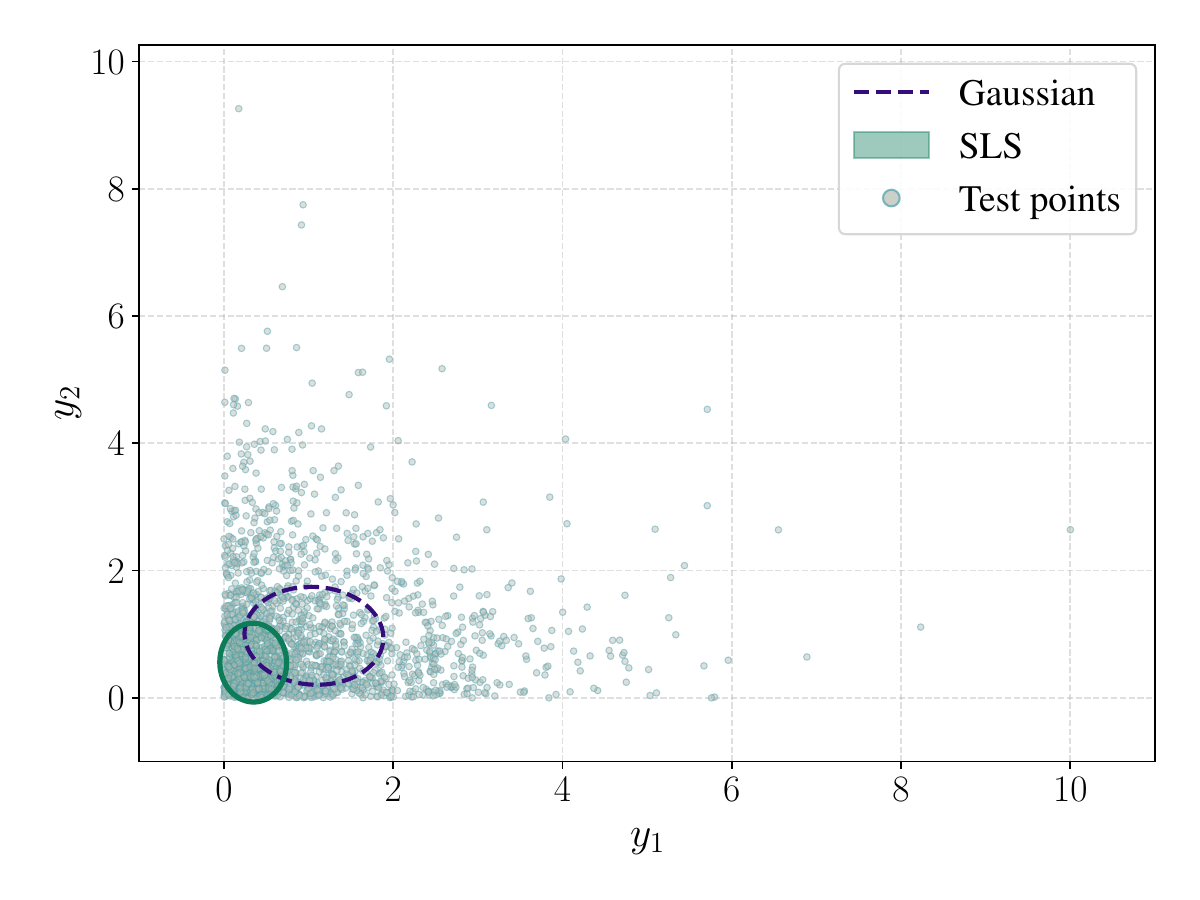}}
\vspace{-5mm}
\caption{\textbf{SLS regression on synthetic 2D distributions for a fixed covariate $X$.} All samples are drawn from the true conditional distribution $\mathbb{P}_{Y|X}$. We compare our approach which uses a single-flow frontier fixed to the identity mapping to enforce rigid ellipsoidal sets, against a standard baseline that minimizes the negative log-likelihood (NLL) of a Gaussian distribution. \textbf{Left:} A Gaussian distribution contaminated with $5\%$ uniform outliers. For a target coverage of $70\%$, our shrinking window method successfully ignores the outliers, whereas the NLL baseline is severely distorted by them. \textbf{Right:} An asymmetric exponential distribution, illustrating a misspecified setting where the true highest density regions are not elliptical. Despite this structural restriction, for a target coverage of $30\%$, our approach extracts a significantly smaller and tighter ellipsoid achieving the desired probability mass compared to the Gaussian baseline.}
\label{fig:comparison:gaussian}
\end{figure}

\subsection{Union of normalizing flows}
\label{app:subsec:union:flow}

To prevent trivial volume collapses, we constrain each shape matrix such that $\det(L_{\theta_k}(X)) = 1$. Under this constraint alone, the resulting predictive region would be restricted to a union of equally sized volumetric shapes. To represent the geometric union of these $K$ flow-transformed ellipsoids while allowing for varying local volumes, we introduce conditional mixture weights $p_k(X)$ satisfying $p_k(X) > 0$ and $\sum_{k=1}^K p_k(X) = 1$. 

We define the individual frontier for the $k$-th component by scaling its latent squared Mahalanobis distance by $p_k(X)^{-2/d}$, where $d$ is the dimensionality of $\mathcal{Y}$:
\begin{equation*}
G_k(X, Y) = \Big\| p_k(X)^{-\frac{1}{d}} L_{\theta_k}(X) \big(T_{\phi_k}(Y; X) - \mu_{\theta_k}(X)\big) \Big\|_2^2\, .
\end{equation*}

This score is equivalent to \Cref{eq:softmin_frontier}, up to the update $L_{\theta_k}(X) \xleftarrow{} p_k(X)^{-\frac{1}{d}} L_{\theta_k}(X)$. This parameterization effectively decouples the global volume allocation from the local geometric shapes dictated by $L_{\theta_k}(X)$. In practice, to prevent premature volume collapse of any single component during early training, we initially freeze the mixture weights uniformly at $p_k(X) = 1/K$ for a set number of epochs. Once the individual flows have stabilized, we unfreeze the weights, allowing the model to dynamically distribute volume across the learned shapes.

\section{Conformalizing the level sets}
\label{app:conformalizing}

In practice, establishing distribution-free finite-sample guarantees for exact conditional coverage is fundamentally impossible without imposing strong assumptions on the underlying data distribution or resorting to trivial prediction sets \citep{foygel2021limits}. While our methodology targets conditional adaptivity by learning a single conditional quantile, a statistically more tractable task than estimating the full conditional distribution, this optimization alone cannot strictly guarantee conditional coverage. To provide rigorous statistical validity, we apply a post-hoc corrective procedure to guarantee marginal coverage \citep{papadopoulos2002inductive}. Specifically, we leverage a split conformal framework \citep{shafer2008tutorial} using the normalized scores $G(X,Y)/q_{\tau}(X)$. This leads to prediction sets of the form \mbox{$C_\tau(X)=\{y \mid G(X,y)\leq r \cdot q_{\tau}(X)\}$}, where $r$ ensures marginal coverage guarantees.

In practice, our method can also be deployed without conformal calibration when the estimator $q_{\tau}$ is sufficiently well learned and exhibits good empirical calibration. In such cases, the learned quantile alone may provide strong conditional adaptivity. The conformal correction serves as a reliable safeguard, acting as a post-hoc adjustment that guarantees marginal coverage even when $q_{\tau}$ is imperfectly estimated.

\section{Architecture}
\label{app:architecture}

In this section, we provide the architectural details of the neural network components used in SLS regression, including the volume-preserving flow, the shape matrix parameterization, the quantile networks, and the shrinking window scheduling.

\subsection{Frontier function architecture}

The frontier function requires parameterizing both a volume-preserving mapping and a localized shape matrix (precision matrix) with a strictly unit determinant.

\paragraph{Conditional volume-preserving flow.}
The flow $T_\phi(Y; X)$ is constructed by composing multiple conditional additive coupling layers. At each layer, the input vector $y \in \mathbb{R}^d$ is partitioned into two components, $y_1$ and $y_2$. To ensure volume preservation (a Jacobian determinant of $1$), we apply an affine transformation where only the translation is parameterized by a neural network (see, e.g., \citealt{dinh2014nice}). Specifically, the forward pass for a single layer is defined as:
\begin{align*}
    z_1 &= y_1 \\
    z_2 &= y_2 + \mathcal{F}_\phi(y_1, X) - \mathcal{F}_\phi(\mathbf{0}, X),
\end{align*}
where $\mathbf{0}$ is a zero vector of the same dimension as $y_1$. Subtracting the component evaluated at the origin acts as a dynamic bias correction. The shift function $\mathcal{F}_\phi$ is parameterized by a multi-layer perceptron (MLP). To stabilize training and prevent overfitting, this MLP utilizes layer normalization \citep{ba2016layer} and dropout \citep{srivastava2014dropout} between its hidden layers, followed by ReLU activations. The final linear projection is initialized with weights and biases set strictly to zero, ensuring that the flow acts as an identity mapping at the beginning of training. Successive layers alternate the masking split so that all dimensions are eventually transformed.

\paragraph{Precision matrix parameterization.}
The shape matrix $L_\theta(X)$ dictates the geometry of the Mahalanobis frontier. We explore two parameterization modes to scale effectively with the dimensionality $d$ of the response space:
\begin{itemize}
    \item \textbf{Full rank:} For moderate dimensions ($d < 5$), we predict all elements of the lower triangular matrix $L_\theta(\cdot)$. The diagonal elements of $L$ are constrained to be strictly positive using a softplus activation. 
    \item \textbf{Low-rank plus diagonal:} For higher dimensions, learning all parameters becomes computationally prohibitive and prone to overfitting. Instead, we use a low-rank approximation, defining the precision matrix as $L^\top L = D + V V^\top$, following the implementation of \cite{braun2025multivariate} for a Gaussian model. Here, $D$ is a positive diagonal matrix and $V \in \mathbb{R}^{d \times r}$ is a factor matrix with rank $r = \lceil \sqrt{d} \rceil$. The Mahalanobis distance in \Cref{eq:flow_frontier} is evaluated directly as the quadratic form $z^\top (D + V V^\top)z$ and the volume surrogate is computed using $\sqrt{\det(D+VV^\top)}$.
\end{itemize}

\subsection{Quantile network}
The conditional quantile estimator $q_\omega(X)$ must simultaneously output the lower bound, the target, and the upper bound of the shrinking window. To encourage representation sharing and reduce computational overhead, we use a single shared MLP backbone. The backbone maps the covariates $X$ to a hidden representation, which is then passed to three independent linear heads: one for $\tau - \phi(n)$, one for the target $\tau$, and one for $\tau + \psi(n)$. Because the frontier function yields strictly positive distance metrics, the outputs of these linear heads are exponentiated (and shifted by a small $\epsilon = 10^{-6}$) to ensure the predicted quantiles are always strictly positive. We then rank the quantile predictions for each feature $X$ to avoid quantile crossing and ensure consistency.

\section{Training strategy}
\label{app:sec:training:strategy}
\subsection{Shrinking window scheduling}
The optimization of the volume surrogate relies heavily on the design of the shrinking margin sequences $\phi(n)$ and $\psi(n)$. In our implementation, we parameterize these boundaries using an annealed logistic schedule. Following an initial warm-start phase of $n_0$ steps where the full volume is penalized, the window begins to shrink from an initial wide margin down to a tight minimum margin. 

Let $t = n - n_0$ be the number of steps post-warm-up. The lower margin function is parameterized as:
\begin{equation*}
    \phi(n) = \text{error}_{\text{init}} + (\text{error}_{\text{min}} - \text{error}_{\text{init}}) \cdot \left( \frac{\sigma(k \cdot (t - t_0)) - \sigma(-k \cdot t_0)}{1 - \sigma(-k \cdot t_0)} \right),
\end{equation*}
where $\sigma(z) = (1 + \exp(-z))^{-1}$ is the sigmoid function, $k$ controls the steepness of the decay, and $t_0$ is the target step center. This normalized sigmoidal decay ensures a smooth, differentiable transition from the exploratory warm-start phase to the strict exploitation of the target quantile. An identical functional form (with independent hyperparameters) is used for the upper bound schedule $\psi(n)$.

We do not force the decay schedules to shrink to exactly zero (i.e., $\phi(n) \to 0$ and $\psi(n) \to 0$). This relaxation of the initial objective is motivated by improved robustness of the learning procedure. If the window were perfectly thin, minor estimation errors in the quantile networks would cause the model to optimize an incorrect threshold. Furthermore, in a practical finite-sample setting, a zero-width window would result in exploding gradients for the data points that would fall within the targeted neighborhood. Instead, we constrain the decay schedules to converge to strictly positive minimum margins, $\phi(n) \to \varepsilon_{\text{low}}$ and $\psi(n) \to \varepsilon_{\text{high}}$. This results in the minimization of the expected volume integrated over a narrow probability band:
\begin{equation*}
\mathbb{E}_X\left[ \int_{\tau -\varepsilon_{\text{low}}}^{\tau+\varepsilon_{\text{high}}} \Vol_G(g, X) \, dF_{G|X}(g) \right],
\end{equation*}
which effectively performs a uniform volume minimization across the localized quantile range $[\tau -\varepsilon_{\text{low}}, \tau + \varepsilon_{\text{high}}]$. 

\subsection{Early stopping}
Because the surrogate objective in \Cref{eq:kn_reformulation} explicitly depends on the training step via the shrinking window, it cannot be reliably used as a criterion for early stopping. Instead, we select the model checkpoint that yields the minimum average volume on a hold-out validation set while satisfying the target marginal coverage. To ensure a fair volumetric comparison across epochs, we first apply a scalar calibration $q(X) \mapsto r \cdot q_\tau(X)$, where $r$ is dynamically computed on the validation set to strictly guarantee marginal coverage, and subsequently measure the resulting adjusted volume.

\subsection{Learning the $\tau$-th quantile}
While our alternating optimization scheme effectively aligns the geometric boundary with the target conditional quantile, the quantile network $q_\tau(X)$ is key to improving conditional coverage. To yield the most accurate final conditional highest density regions, once $G$ has been learned, we perform a learning step using TabICL \citep{qu2026tabiclv2} on the training set. By leveraging TabICL on a the training set, we refine the final boundary threshold, ensuring a highly accurate estimate of the conditional quantile before extracting the final level set. To obtain a more accurate estimate of the true conditional quantile, one could instead learn this quantity on a held-out dataset. However, to avoid introducing additional methodological complexity and to keep the framework streamlined, we do not adopt this approach. As a consequence, the frontier function $G$ may exhibit some degree of overfitting to the training data.

\section{Baselines}
\label{app:baselines}

Among our baseline choices, we choose all strategies targeting conditional coverage from \cite[Table 1]{dheur2025unified} in multivariate conformal prediction. The base model for conditional density estimation minimizes the negative log-likelihood of a multivariate quantile function forecaster \citep{kan2022multivariate}, utilizing input convex neural networks \citep{brandon2017input} as implemented by \citet{dheur2025unified}.
\begin{itemize}
\item \textbf{C-HDR} \citep{izbicki2022cd}: Conformalizes the highest predictive density (HPD) by using the nonconformity score $S_\text{HDR}(X, Y) =  \P_{y \sim \hat{p}(\cdot|X)}\left( \; \hat{p}(y|X) \geq \hat{p}(Y|X) \; \right)$. It then produces regions \mbox{$ C_{\text{C-HDR}}(X) = \{y : \hat p(y|X) \ge \hat t_q\}$}, where \( \hat t_q \) defines the highest density region (HDR) at level \( \hat q \).

\item \textbf{CP2-PCP} \citep{plassier2024probabilistic}: Builds predictive sets by using samples from an implicit conditional generative model. For each calibration point it uses two independent draws from the conditional generator to define a conformity score and an inflation parameter $\tau$ that accounts for the conditional mass around likely outputs. At prediction time it forms a union of balls around new generated samples, with their size chosen to guarantee marginal validity while improving approximate conditional adaptivity. 

\item \textbf{L-CP} \citep{dheur2025unified}: Defines conformity in a latent space using an invertible conditional generative model \( \hat Q: \mathcal{Z} \times \mathcal{X} \to \mathcal{Y} \). A latent variable \( Z \sim \mathcal{N}(0, I_d) \) is mapped to the output space via \( \hat Q \), and the conformity score is measured in latent space as
\[
S_{\text{L-CP}}(X,Y) = \| \hat Q^{-1}(Y; X) \|.
\]
The prediction region is obtained by taking a ball of radius \( \hat q \) around the origin in latent space and mapping it back to the output space. This method avoids grid-based directional quantile regression, improving scalability and computational efficiency, and generalizes distributional conformal prediction to multivariate outputs.

\item \textbf{C-PCP} \citep{dheur2025unified}: Estimates the conditional CDF of the conformity score \( S(X,Y) \),
\[
S_{\text{CDF}}(x,y) = \P(S_W(X,Y) \le S_W(x,y) \mid X = x),
\]
using a Monte Carlo approximation with $K$ samples
\[
S_{\text{ECDF}}(x,y) = \frac{1}{K} \sum_{k =1}^K \mathds{1}[S_W(x,\hat Y^{(k)}) \le S_W(x,y)],
\quad \hat Y^{(k)} \sim \hat F_{Y|x}.
\]
When \( S(x,y) = S_{\text{PCP}}(x,y) \), this yields
\[
S_{\text{C-PCP}}(x,y) = \frac{1}{K} \sum_{k \in [K]} 
1\!\left\{
\min_{l \in [L]} \|\hat Y^{(k)} - \tilde Y^{(l)}\|_2 \le
\min_{l \in [L]} \|y - \tilde Y^{(l)}\|_2
\right\}.
\]

\item \textbf{PCP} \citep{wang2023probabilistic}: Draws \( L \) samples \( \tilde Y^{(l)} \sim \hat p_{Y|x} \) with $\hat p_{Y|x}$ the estimated conditional distribution, and defines conformity as the distance to the nearest sample, \mbox{\( S_{\text{PCP}}(X,Y) = \min_{l\in[L]} \|Y - \tilde Y^{(l)}\|_2 \)}; the corresponding region is a union of \( L \) balls centered at the sampled points.

\end{itemize}

\section{Additional experiments}
\label{app:additional:experiments}

\paragraph{Hardware.} We ran the models on GPUs (NVIDIA V100).

\subsection{Minimizing alternative functions of the conditional quantile}
\label{app:subsec:function:of:quantile}

To demonstrate the generality of our framework, we evaluate its ability to minimize functions of the conditional quantile other than the induced volume. Specifically, following the experimental setup from \Cref{sec:experiments}, we adapt our approach to minimize the expected conditional median absolute deviation:
\begin{equation}
\label{eq:median_example}
\inf_{f:\mathcal{X} \to \mathbb{R}} \mathbb{E}_{X}\Big[ \operatorname{median}_{Y\sim \P_{Y|X}}\big(|Y-f(X)|\big) \Big]\,.
\end{equation}

We compare our method against two alternatives: (1) a naive batch approach that optimizes the objective solely over the data points responsible for the empirical median, and (2) a Conditional Value-at-Risk (CVaR) approach, which uses a bi-level optimization similar to ours but replaces the shrinking window with a strict tail indicator, $K(X, G(X,Y)) = \mathbbm{1}_{\{G(X,Y) \geq q(X)\}}$. Because the true data-generating process is known in this synthetic setup, we also compute and display the theoretically optimal oracle strategy. Results are available in \Cref{fig:median_comparison}. 
These alternative are not proper baselines for this task, since they were not designed for it, but are just current alternatives (in particular, the goal of CVaR is to minimize the error on the outliers, which explains why it is not well adapted for this task).

Overall, our approach successfully and accurately estimates the true conditional median absolute deviation. In contrast, the CVaR baseline struggles in this asymmetric noise setting, as its tail-focused objective disproportionately penalizes distances to outliers. While the marginal relaxation approach initially appears effective when the data distribution is relatively well-behaved, the marginal nature of its predictions ultimately leads to vanishing gradients in the exact regions where the model makes the largest errors, a limitation clearly visible in the right-hand scenario.

\begin{figure}[h]
\center
\subfigure{\includegraphics[width=0.48\columnwidth]{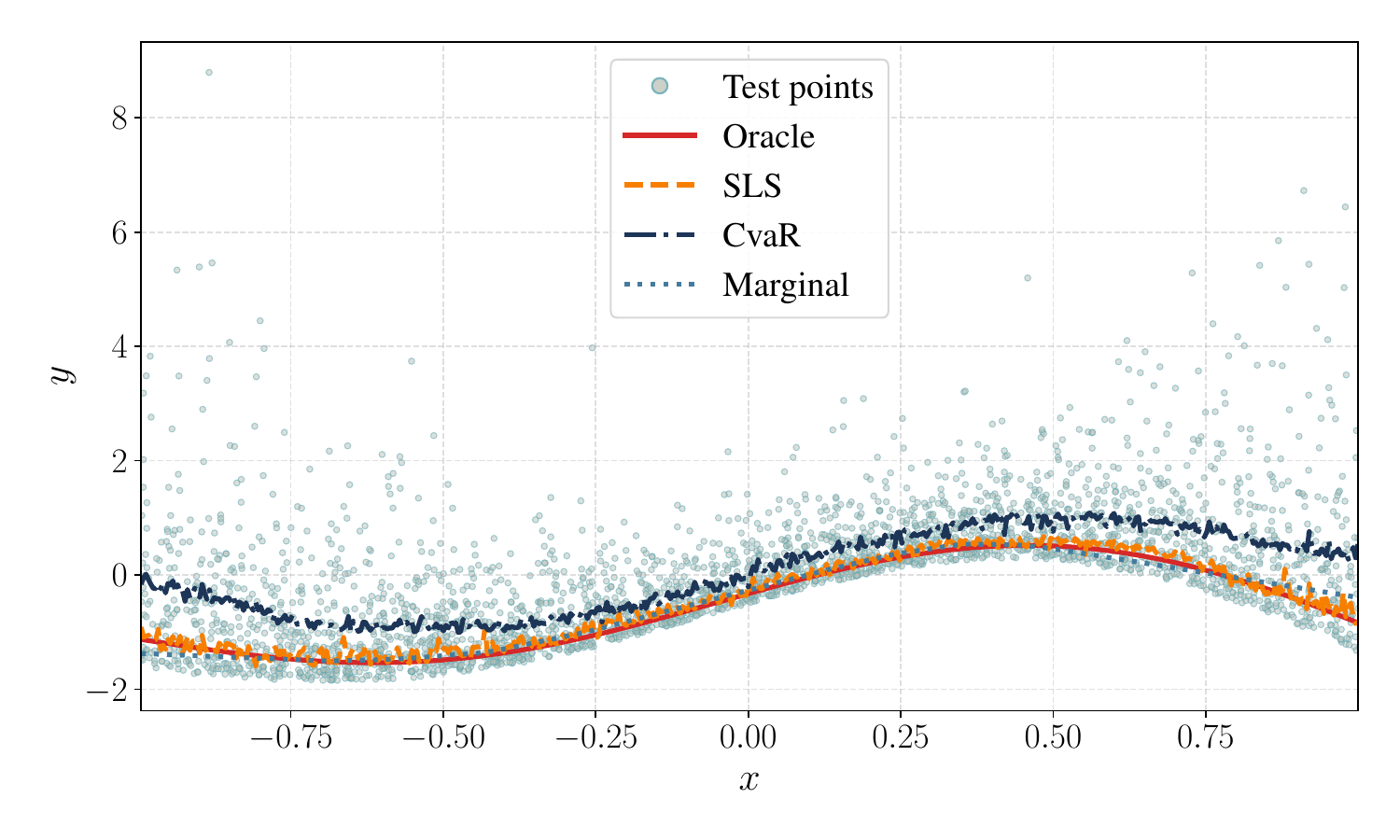}} ~
\subfigure{\includegraphics[width=0.48\columnwidth]{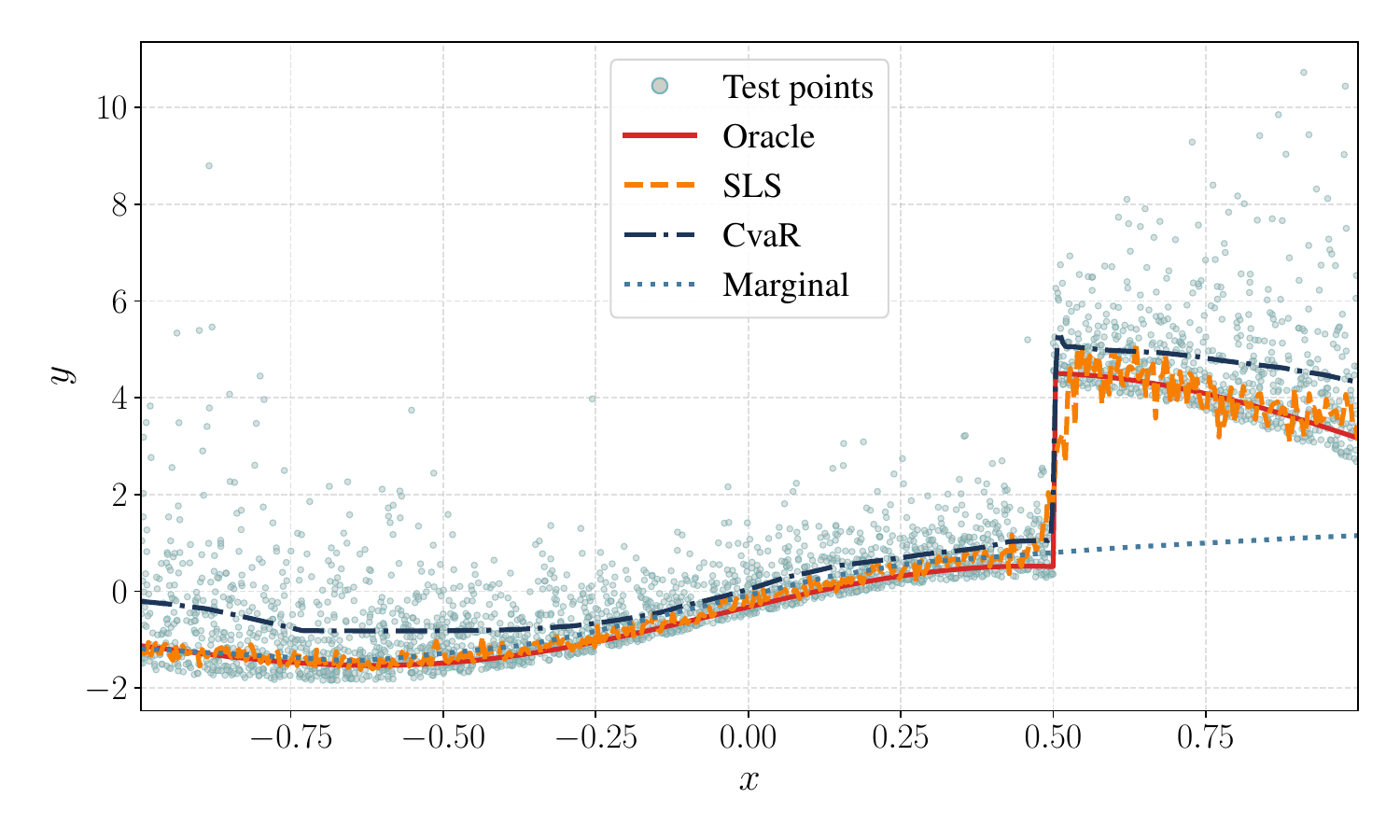}}
\vspace{-5mm}
\caption{\textbf{Minimization of the conditional median absolute deviation (\Cref{eq:median_example}).} The data are generated as $Y = f(X) + r(X)W$, where $W \sim \mathcal{E}(1) - 1$, $X \sim \mathcal{U}([-1,1])$, and heteroscedasticity is introduced via $r(X) = X^2 + 0.5$. \textbf{Left:} Continuous base trend with $f(X) = \sin(3X)$. \textbf{Right:} Discontinuous base trend with $f(X) = \sin(3X) + 4\mathbbm{1}_{\{X>0.5\}}$.}
\label{fig:median_comparison}
\end{figure}

\subsection{Real data}

\paragraph{Dataset information.}

Data are pre-processed using a quantile transformation on both the features ($X$) and responses ($Y$) using \texttt{scikit-learn} \citep{pedregosa2011scikit}, ensuring normalization across datasets. See \Cref{tab:dataset:multi:information} for specific information on the datasets.

\begin{table}[h!]
\caption{Description of the multivariate datasets.}
\centering
\begin{tabular}{@{}lccc@{}}
\toprule
Dataset & \shortstack{Number of \\ samples}  & \shortstack{Number \\ of features} & \shortstack{Dimension \\ of targets}  \\
\midrule
Bias \citep{cho2020comparative} & 7752 & 22 & 2  \\ \hline
CASP \citep{rana2013physicochemical} & 45730 & 8 & 2  \\ \hline
House \citep{pace1997sparse} & 21613 & 17 & 2 \\ \hline
rf1 \citep{tsoumakas2011mulan} & 9125 & 64 & 8 \\ \hline
rf2 \citep{tsoumakas2011mulan} & 9125 & 576 & 8 \\ \hline
Taxi \citep{wang2023conformal} & 61286 & 6 & 2 \\ \hline
\bottomrule
\end{tabular}
\label{tab:dataset:multi:information}
\end{table}

\subsection{Results per dataset}

For these experiments, we utilized a single-flow-based score with the model introduced in \Cref{app:architecture} with 3 layers. The multi-flow approach did not yield performance improvements, suggesting that the underlying distributions of these datasets are largely unimodal. The normalizing flows are parameterized using three layers. Regarding the covariance parameterization, we adopted a full-rank structure for datasets with fewer than five outputs ($d < 5$), and a low-rank approximation with rank $r=\lceil \sqrt{d} \rceil$ otherwise, where $d$ is the output dimension. For output dimensions $d\geq5$, we minimize the $\log$ volume instead of the volume. All reported results are averaged across 10 independent trials. To prevent statistical anomalies from skewing the aggregated results, we discard extreme outlier values. Notably, while this instability never occurs with SLS regression, it can manifest in baselines that rely on sampling procedures, such as C-PCP or CP2-PCP.

The performance outcomes are summarized for target coverage rates $\tau = 0.9$ in \Cref{table:results:alpha:0.1} and $\tau = 0.5$ in \Cref{table:results:alpha:0.5}. These tables report two primary metrics: the volume of the prediction sets and the conditional coverage, evaluated via the ERT metric \citep{braun2025conditional} using their default LightGBM classifier. Comparing the efficacy of the different strategies presents an inherent challenge due to the strict trade-off between efficiency and validity: a method achieving a smaller volume often does so at the cost of degraded conditional coverage. Consequently, both the volume and the ERT values must be considered jointly to accurately assess a model's performance.

\begin{table}[htpb]
\centering
\caption{Performance metrics (mean and standard error) per dataset for $\tau=0.9$.}
\resizebox{\textwidth}{!}{
\begin{tabular}{ll|cccccc}
\toprule
Dataset & Metric & SLS & C-HDR & PCP & CP2-PCP & L-CP & C-PCP \\ \midrule
\multirow{3}{*}{biais} & Volume & $\mathbf{1.11}_{{0.02}}$ & $1.34_{{0.03}}$ & $1.35_{{0.01}}$ & $1.42_{{0.02}}$ & $1.35_{{0.03}}$ & $1.47_{{0.02}}$ \\
 & ERT & $0.04_{{0.00}}$ & $0.02_{{0.01}}$ & $\mathbf{0.01}_{{0.01}}$ & $0.04_{{0.00}}$ & $0.03_{{0.01}}$ & $0.02_{{0.00}}$ \\
\midrule
\multirow{3}{*}{rf1} & Volume & $\mathbf{0.36}_{{0.00}}$ & $0.46_{{0.01}}$ & $0.49_{{0.01}}$ & $0.49_{{0.00}}$ & $0.48_{{0.01}}$ & $0.49_{{0.00}}$ \\
 & ERT & $0.04_{{0.00}}$ & $0.05_{{0.00}}$ & $0.09_{{0.00}}$ & $0.07_{{0.00}}$ & $0.05_{{0.01}}$ & $\mathbf{0.04}_{{0.00}}$ \\
\midrule
\multirow{3}{*}{rf2} & Volume & $0.47_{{0.01}}$ & $0.44_{{0.00}}$ & $0.44_{{0.01}}$ & $0.98_{{0.53}}$ & $\mathbf{0.43}_{{0.01}}$ & $0.44_{{0.01}}$ \\
 & ERT & $0.05_{{0.00}}$ & $0.05_{{0.00}}$ & $0.07_{{0.01}}$ & $0.07_{{0.00}}$ & $0.05_{{0.00}}$ & $\mathbf{0.04}_{{0.00}}$ \\
\midrule
\multirow{3}{*}{casp} & Volume & $\mathbf{1.15}_{{0.01}}$ & $1.32_{{0.01}}$ & $1.48_{{0.01}}$ & $1.45_{{0.01}}$ & $1.42_{{0.01}}$ & $1.51_{{0.01}}$ \\
 & ERT & $0.02_{{0.00}}$ & $0.02_{{0.00}}$ & $0.02_{{0.00}}$ & $0.04_{{0.01}}$ & $0.02_{{0.00}}$ & $\mathbf{0.01}_{{0.00}}$ \\
\midrule
\multirow{3}{*}{house} & Volume & $1.20_{{0.01}}$ & $\mathbf{1.15}_{{0.01}}$ & $1.22_{{0.01}}$ & $1.42_{{0.15}}$ & $1.27_{{0.01}}$ & $1.80_{{0.51}}$ \\
 & ERT & $0.03_{{0.00}}$ & $0.03_{{0.00}}$ & $0.04_{{0.00}}$ & $0.03_{{0.00}}$ & $0.02_{{0.00}}$ & $\mathbf{0.02}_{{0.00}}$ \\
\midrule
\multirow{3}{*}{taxi} & Volume & $2.77_{{0.02}}$ & $\mathbf{2.65}_{{0.11}}$ & $3.46_{{0.28}}$ & $5.26_{{0.79}}$ & $3.08_{{0.02}}$ & $3.25_{{0.01}}$ \\
 & ERT & $0.01_{{0.00}}$ & $0.01_{{0.00}}$ & $0.01_{{0.00}}$ & $0.03_{{0.02}}$ & $\mathbf{0.01}_{{0.00}}$ & $0.01_{{0.00}}$ \\
 \bottomrule
\end{tabular}
}
\label{table:results:alpha:0.1}
\end{table}

\begin{table}[htpb]
\centering
\caption{Performance metrics (mean and standard error) per dataset for $\tau=0.5$.}
\resizebox{\textwidth}{!}{
\begin{tabular}{ll|cccccc}
\toprule
Dataset & Metric & SLS & C-HDR & PCP & CP2-PCP & L-CP & C-PCP \\ \midrule
\multirow{3}{*}{bias} & Volume & $\mathbf{0.59}_{{0.01}}$ & $0.69_{{0.01}}$ & $0.77_{{0.01}}$ & $0.79_{{0.01}}$ & $0.70_{{0.01}}$ & $0.80_{{0.01}}$ \\
 & ERT & $0.06_{{0.00}}$ & $0.07_{{0.00}}$ & $0.08_{{0.00}}$ & $\mathbf{0.04}_{{0.01}}$ & $0.07_{{0.01}}$ & $0.05_{{0.01}}$ \\
\midrule
\multirow{3}{*}{rf1} & Volume & $\mathbf{0.26}_{{0.00}}$ & $0.29_{{0.00}}$ & $0.32_{{0.01}}$ & $0.36_{{0.00}}$ & $0.31_{{0.00}}$ & $0.36_{{0.00}}$ \\
 & ERT & $\mathbf{0.07}_{{0.00}}$ & $0.22_{{0.01}}$ & $0.29_{{0.00}}$ & $0.12_{{0.00}}$ & $0.22_{{0.00}}$ & $0.13_{{0.01}}$ \\
\midrule
\multirow{3}{*}{rf2} & Volume & $0.35_{{0.01}}$ & $\mathbf{0.28}_{{0.00}}$ & $0.30_{{0.01}}$ & $0.33_{{0.01}}$ & $0.28_{{0.01}}$ & $0.33_{{0.01}}$ \\
 & ERT & $\mathbf{0.08}_{{0.00}}$ & $0.20_{{0.01}}$ & $0.26_{{0.01}}$ & $0.12_{{0.01}}$ & $0.20_{{0.00}}$ & $0.12_{{0.01}}$ \\
\midrule
\multirow{3}{*}{house} & Volume & $0.66_{{0.01}}$ & $\mathbf{0.52}_{{0.00}}$ & $0.63_{{0.01}}$ & $0.68_{{0.01}}$ & $0.63_{{0.01}}$ & $0.68_{{0.01}}$ \\
 & ERT & $0.05_{{0.00}}$ & $0.09_{{0.00}}$ & $0.13_{{0.00}}$ & $0.04_{{0.00}}$ & $0.08_{{0.00}}$ & $\mathbf{0.04}_{{0.00}}$ \\
\midrule
\multirow{3}{*}{casp} & Volume & $0.60_{{0.00}}$ & $\mathbf{0.47}_{{0.06}}$ & $0.63_{{0.06}}$ & $0.72_{{0.01}}$ & $0.56_{{0.06}}$ & $0.80_{{0.14}}$ \\
 & ERT & $\mathbf{0.03}_{{0.00}}$ & $0.11_{{0.00}}$ & $0.15_{{0.00}}$ & $0.05_{{0.00}}$ & $0.10_{{0.00}}$ & $0.14_{{0.06}}$ \\
\midrule
\multirow{3}{*}{taxi} & Volume & $1.32_{{0.01}}$ & $\mathbf{1.27}_{{0.01}}$ & $1.50_{{0.10}}$ & $1.58_{{0.13}}$ & $1.41_{{0.15}}$ & $1.56_{{0.14}}$ \\
 & ERT & $0.02_{{0.00}}$ & $0.21_{{0.12}}$ & $0.20_{{0.08}}$ & $\mathbf{0.01}_{{0.00}}$ & $0.03_{{0.00}}$ & $0.33_{{0.11}}$ \\
\bottomrule
\end{tabular}
}
\label{table:results:alpha:0.5}
\end{table}

\end{appendices}

\end{document}